\documentclass{lmcs}
\pdfoutput=1

\usepackage{lastpage}
\lmcsdoi{19}{3}{13}
\lmcsheading{}{\pageref{LastPage}}{}{}%
{Sep.~20,~2022}{Aug.~10,~2023}{}

\usepackage[utf8]{inputenc}


\usepackage{bbding}
\usepackage{subcaption}
\usepackage{graphicx}

\usepackage{amsmath,amssymb}
\usepackage{thmtools}
\usepackage{thm-restate}
\usepackage{cite}
\usepackage[all]{foreign}
\usepackage[bookmarks,unicode,colorlinks=true]{hyperref}
\usepackage[inline]{enumitem}
\usepackage[capitalise]{cleveref}
\usepackage{varwidth}
\usepackage{tikz}
\tikzset{
  alt/.code args={<#1>#2#3}{%
    \alt<#1>{\pgfkeysalso{#2}}{\pgfkeysalso{#3}}%
  },
  only/.code args={<#1>#2}{%
    \alt<#1>{\pgfkeysalso{#2}}{}%
  },
  invisible/.style={opacity=0},
  visible on/.style={alt=#1{}{invisible}},
  alerted on/.style={alt=#1{alerted}{}},
  math mode/.style={execute at begin node=$, execute at end node=$},
  math array with columns/.style={
    execute at begin node=\(\begin{array}{c},
    execute at end node=\end{array}\)
  },
  math array/.style={math array with columns=c},
  square/.style={regular polygon,regular polygon sides=4}
}

\usepackage{packages/commands}
\usepackage{packages/symbols}
\usepackage{packages/logic}

\NewDocumentCommand\ltlset{m}{\ensuremath{\set{\ltl{#1}}}}
\newtheorem{observation}[thm]{Observation}

\crefname{thm}{Theorem}{Theorems}
\crefname{cor}{Corollary}{Corollaries}
\crefname{lem}{Lemma}{Lemmas}
\crefname{prop}{Proposition}{Propositions}
\crefname{rem}{Reminder}{Reminders}
\crefname{defi}{Definition}{Definitions}

\makeatletter
\RequirePackage[bookmarks,unicode,colorlinks=true]{hyperref}%
   \def\@citecolor{blue}%
   \def\@urlcolor{blue}%
   \def\@linkcolor{blue}%

\def\orcidID#1{\smash{\href{http://orcid.org/#1}{\protect\raisebox{-1.25pt}{\protect\includegraphics{ORCID_Color.eps}}}}}
\makeatother

\tikzset{
  math mode/.style={execute at begin node=\(,execute at end node=\)},
  tableau edge/.style={-, thin, solid},
  tableau node/.style={preprocess node content=\ltlset},
  step rule/.style={->,thick},
  state/.style={fill, circle, inner sep=2pt},
  sibling distance=4cm,
  level distance=1cm
}

\title[A FO logic char. of safety and co-safety languages]{%
  A first-order logic characterization of\texorpdfstring{\\}{ }safety and co-safety languages%
}

\author[A.~Cimatti]{Alessandro Cimatti\lmcsorcid{0000-0002-1315-6990}}[a]
\author[L.~Geatti]{Luca Geatti\lmcsorcid{0000-0002-7125-787X}}[b]
\author[N.~Gigante]{Nicola Gigante\lmcsorcid{0000-0002-2254-4821}}[c]
\author[A.~Montanari]{Angelo Montanari\lmcsorcid{0000-0002-4322-769X}}[b]
\author[S.~Tonetta]{Stefano Tonetta\lmcsorcid{0000-0001-9091-7899}}[a]
 
\address{Fondazione Bruno Kessler, Via Sommarive, 18, Povo, Trento, 38123, Italy}	
\email{\{cimatti,tonettas\}@fbk.eu}

\address{University of Udine, Via delle Scienze 206, Udine, 33100, Italy} 
\email{\{luca.geatti,angelo.montanari\}@uniud.it}

\address{Free University of Bozen-Bolzano, Piazza Università, 1, Bolzano, 39100, Italy}	
\email{gigante@inf.unibz.it}

\begin{document}

\maketitle

\begin{abstract}
  \emph{Linear Temporal Logic} (\LTL) is one of the most popular temporal
  logics and comes into play in a variety of branches of computer
  science.
  Among the various reasons of its widespread use there are its strong foundational properties:
  \LTL is equivalent to counter-free $\omega$-automata, to star-free
  $\omega$-regular expressions, and (by Kamp's theorem) to the
  \emph{First-Order Theory of Linear Orders} (\FOTLO).
  Safety and co-safety languages, where a finite prefix suffices to
  establish whether a word does not belong or belongs to the language,
  respectively, play a crucial role in lowering the complexity of problems
  like model checking and reactive synthesis for \LTL.
  \safetyltl (resp., \cosafetyltl) is a fragment of \LTL where only the
  \emph{tomorrow}, the \emph{weak tomorrow} and the \emph{until} temporal
  modalities (resp., the \emph{tomorrow}, the \emph{weak tomorrow} and the
  \emph{release} temporal modalities) are allowed, that recognises safety
  (resp., co-safety) languages only.

  The main contribution of this paper is the introduction of a fragment of
  \FOTLO, called \SafetyFO, and of its dual \coSafetyFO, which are
  \emph{expressively complete} with respect to the \LTL-definable safety
  and co-safety languages.
  We prove that they exactly characterize
  \safetyltl and \cosafetyltl, respectively, a result that joins Kamp's theorem, and
  provides a clearer view of the characterization of (fragments of) \LTL in
  terms of first-order languages.
  In addition, it gives a direct, compact, and self-contained proof that
  any safety language definable in \LTL is definable in \safetyltl as well.
  As a by-product, we obtain some interesting results on the expressive
  power of the \emph{weak tomorrow} operator of \safetyltl, interpreted over
  finite and infinite words. Moreover, we prove that, when interpreted
  over finite words, \safetyltl (resp. \cosafetyltl) devoid of the \emph{tomorrow} (resp., \emph{weak
  tomorrow}) operator captures the safety (resp., co-safety) fragment of
  \LTL over finite words.

  We then investigate some formal properties of \SafetyFO and \coSafetyFO:
  \begin{enumerate*}[label=(\roman*)]
    \item we study their succinctness with respect to their modal
      counterparts, namely, \safetyltl and \cosafetyltl;
    \item we illustrate an important practical application of them in the
      context of reactive synthesis;
    \item we compare them with expressively equivalent first-order
      fragments.
  \end{enumerate*}
  
  Last but not least, we provide different characterizations of the
  (co-)safety fragment of \LTL in terms of temporal logics, automata, and
  regular expressions.

  %
\end{abstract}


\section{Introduction}
\label{sec:intro}

\emph{Linear Temporal Logic} (\LTL) is the de-facto standard logic for
system specifications~\cite{pnueli1977temporal}. It is a modal logic that is usually
interpreted over infinite state sequences, but the finite-words semantics has
recently gained attention as well~\cite{DeGiacomoV13,DeGiacomoV15}. 
The widespread use of \LTL is due to its simple syntax and semantics,  and to 
its strong foundational properties. Among them, we would like to  mention the 
seminal work by Kamp \cite{kamp1968tense} and Gabbay \etal \cite{gabbay1980temporal} 
on its expressive completeness, that is, \LTL-definable languages 
are exactly those definable in the first-order fragment of the monadic
second-order theory of linear orders \cite{buchi1990decision} (\FOTLO
for short).

In formal verification, an important class of specifications is
that of \emph{safety languages}. They are languages of infinite words
where a finite prefix suffices to establish whether a word does not belong to
the language. As an example, the set of all and only those infinite sequences where some particular bad event never happens can be regarded as a safety language. 
In the dual \emph{co-safety} languages (sometimes called \emph{guarantee} languages), a finite prefix is sufficient to tell whether a word \emph{belongs} to the language, \eg when some desired event is mandated to eventually happen.
Safety and co-safety languages are important for verification, model-checking,
monitoring, and automated synthesis, because they capture a variety of real-world
requirements while being much simpler to deal with
algorithmically~\cite{kupferman2001model,biere2002liveness,ZhuTLPV17}.

\safetyltl is the fragment of \LTL where only the \emph{tomorrow}, the
\emph{weak tomorrow} and the \emph{until} temporal modalities are allowed.
Similarly, its dual \cosafetyltl is obtained by only allowing the
\emph{tomorrow}, the \emph{weak tomorrow} and the \emph{release}
modalities. It has been proved by Chang \etal~\cite{ChangMP92} that
\safetyltl and \cosafetyltl define exactly the safety and co-safety
languages that are definable in \LTL, respectively.

The paper consists of four parts.

In the first part, we provide a novel characterization of \LTL-definable safety
languages, and of their duals, in terms of a fragment of \FOTLO, called
\SafetyFO, and of its dual \coSafetyFO. We argue that they have a very natural
syntax, and we prove that they are \emph{expressively complete} with respect to
\LTL-definable safety and co-safety languages.
We first prove the correspondence between \coSafetyFO and \cosafetyltl,
which extends naturally to their duals and can be viewed as a version
of Kamp's theorem \cite{kamp1968tense} specialized for safety and co-safety
properties. Such a result provides a clearer picture of the correspondence
between (fragments of) temporal and first-order logics.
Then, we exploit it to prove the correspondence between co-safety languages
definable in \LTL and \coSafetyFO, thus establishing also the equivalence
between the former and \cosafetyltl.
This gives a new proof of the fact that \safetyltl captures exactly the set of
\LTL-definable safety languages~\cite{ChangMP92}, which can be viewed as
another contribution of the paper.

The interest of the latter proof is twofold: on the one hand, the original proof by
Chang \emph{et~al.}~\cite{ChangMP92} is only sketched and it relies on two
non-trivial translations scattered across different
sources~\cite{zuck1986past,ShermanPH84}; on the other hand, such an equivalence
result seems not to be very much known, as some authors presented the problem as
open as lately as 2021~\cite{ZhuTLPV17,de2021finite}. 
Thus, a compact and self-contained proof of the result seems to be a useful
contribution for the community. It is worth to note that both proofs build
on the fact that safety/co-safety languages can be captured by formulas of
the form $\ltl{G\alpha}$/$\ltl{F\alpha}$ with $\alpha$ pure-past, but, after
that, the two proofs significantly diverge.
At the end of this part, as a by-product, 
we give some results that assess the expressive power of the \emph{weak tomorrow}
operator of \safetyltl when interpreted over finite \emph{vs.}\ infinite
words.


The second part 
is devoted to the safety and co-safety
fragments of \LTL interpreted over finite words. We show that the logic
obtained from \safetyltl (resp. \cosafetyltl) by forbidding the \emph{tomorrow} (resp.,
\emph{weak-tomorrow}) operator captures the set of safety (resp., co-safety)
properties of \LTL over finite words. This provides a clearer view of
which fragments of \LTL over finite words characterize the safety and
co-safety fragments.

In the third part, 
we study some formal properties of \coSafetyFO and \SafetyFO.  We begin by studying 
the succinctness of \coSafetyFO with respect to \cosafetyltl. We first show 
that there is a linear-size equivalence-preserving translation from \cosafetyltl to
\coSafetyFO. Then, we show 
that the proposed translation from \cosafetyltl to \coSafetyFO 
that we exploit to prove 
the expressive equivalence between the two formalisms is nonelementary.  
Next, we illustrate an interesting practical application of \coSafetyFO to reactive synthesis from temporal specifications.  Finally, we
compare \coSafetyFO with another fragment of \FOTLO that has been proved
to be expressively equivalent to the co-safety fragment of
\LTL~\cite{thomas1988safety}. Naturally, all the above results can be
dualized for the case of \SafetyFO.

In the fourth and last part, 
we summarize the other
characterizations of the (co)safety fragment of \LTL that have been
proposed in the literature so far, that is, those in terms of
\begin{enumerate*}[label=(\roman*)]
  \item temporal logics,
  \item automata, and
  \item regular expressions.
\end{enumerate*}

The paper is organized as follows. \cref{sec:preliminaries} provides some
background knowledge. 
\cref{sec:fragments} introduces \SafetyFO and \coSafetyFO and proves their
correspondence with \safetyltl and \cosafetyltl, respectively.  Then, \cref{sec:proof}
proves their correspondence with the set of safety and co-safety languages
definable in \LTL, thus providing a compact and self-contained proof of the
equivalence between \safetyltl and \LTL-definable safety languages. Some
properties of the \emph{weak next} operator are outlined as well.
\cref{sec:finite} proves the expressive completeness of the
fragment of \LTL over finite words devoid of the \emph{tomorrow}
(resp., \emph{weak-tomorrow}) operator and the safety (resp., co-safety)
fragment of \LTL over finite words.
\cref{sec:comp} compares \coSafetyFO with related fragments and describes
a practical application of \coSafetyFO to reactive synthesis.
\cref{sec:char} summarizes the state of the art about different
characterizations of the (co)safety fragment of \LTL. Finally, \cref{sec:conclusions} 
provides an assessment of the work done and discusses future work.

The paper is a revised and largely extended version of~\cite{DBLP:conf/fossacs/CimattiGGMT22}. In particular, the whole second, third, and 
fourth parts of the paper were not present in~\cite{DBLP:conf/fossacs/CimattiGGMT22}.


\section{Preliminaries}
\label{sec:preliminaries}

\* 
  Scaletta:
  * Safety and co-safety languages
  * LTL, SafetyLTL, coSafetyLTL
  * First-order logic over words
  * Known results
*/

Let $A$ be a finite alphabet. We denote by $A^*$ and $A^\omega$ the set of all
finite and infinite words over $A$, respectively. Moreover, we let
$A^+=A^*\setminus\set{\epsilon}$, where $\epsilon$ is the empty word. Given a
word $\sigma\in A^*$, we denote by $|\sigma|$ the length of $\sigma$. For an
infinite word $\sigma\in A^\omega$, $|\sigma|=\omega$. Given a (finite or
infinite) word $\sigma$, we denote by $\sigma_i\in A$, for $0\le i < |\sigma|$,
the letter at the $i$-th position of the word. For $0\le i\le j < |\sigma|$,
we denote by $\sigma_{[i,j]}$ the subword that starts at the $i$-th position (letter)
 of the word and ends at the $j$-th one, extrema included. By $\sigma_{[i,\infty]}$ we
denote the suffix of $\sigma$ starting at the $i$-th position. Given a word
$\sigma\in A^*$ and $\sigma'\in A^*\cup A^\omega$, we denote the
\emph{concatenation} of the two words as $\sigma\cdot\sigma'$, or simply
$\sigma\sigma'$. A \emph{language} $\lang$, with $\lang\subseteq A^*$ or
$\lang\subseteq A^\omega$, is a set of words. Given two languages $\lang$ and
$\lang'$ with $\lang\subseteq A^*$ and either $\lang'\subseteq A^*$ or
$\lang'\subseteq A^\omega$, we define $\lang\cdot\lang'$ as the set $\set{\sigma\cdot\sigma'
\suchthat \text{$\sigma\in\lang$ and $\sigma'\in\lang'$}}$. Given a finite word
$\sigma=\sigma_0\ldots \sigma_k$, let $\sigma^r=\sigma_k\ldots \sigma_0$ be the
reverse of $\sigma$, and given a language of finite words $\lang$, let
$\lang^r=\set{\sigma^r\suchthat \sigma\in \lang}$. We are now ready to define
\emph{safety} and \emph{co-safety} languages.

\begin{defi}[Safety language \cite{kupferman2001model,thomas1988safety}]
  \label{def:safelang}
  Let $\lang \subseteq A^\omega$. We say that $\lang$ is a \emph{safety
  language} if and only if for all 
  $\sigma \in A^\omega$, it holds
  that if $\sigma \not \in \lang$, then there exists $i\in\N$ such that, for
  all $\sigma'\in A^\omega$, $\sigma_{[0,i]}\cdot\sigma' \not \in \lang$. The 
  class of safety languages is denoted by \SAFETY.
\end{defi}

\begin{defi}[Co-safety language \cite{kupferman2001model,thomas1988safety}]
  \label{def:cosafelang}
  Let $\lang \subseteq A^\omega$. We say that $\lang$ is a \emph{co-safety
  language} if and only if for all 
  $\sigma \in A^\omega$, it holds
  that if $\sigma \in \lang$, then there exists $i\in\N$ such that, for
  all $\sigma'\in A^\omega$, $\sigma_{[0,i]}\cdot\sigma' \in \lang$. The 
  class of co-safety languages is denoted by \coSAFETY.
\end{defi}

\emph{Linear Temporal Logic with Past} (\LTLP) is a temporal logic interpreted over
infinite or finite words. Given a set  of proposition letters $\Sigma$, the
set of \LTLP formulas $\phi$ is generated by the following grammar:
\begin{align*}
  \phi ::= p & \choice \neg\phi \choice \phi\lor\phi \choice 
                           \phi\land\phi  & \text{Boolean connectives} \\
                   & \choice \ltl{X\phi} \choice \ltl{wX\phi} 
                     \choice \ltl{\phi U \phi}
                     \choice \ltl{\phi R \phi} & \text{future modalities} \\
                   & \choice \ltl{Y\phi} \choice \ltl{Z\phi} 
                     \choice \ltl{\phi S \phi}
                     \choice \ltl{\phi T \phi} & \text{past modalities}
\end{align*}
where $p\in\Sigma$ and $\phi$ is an \LTLP formula. We say that
an \LTLP formula is a \emph{pure future} formula if it does not make use of past
modalities, and that it is \emph{pure past} if it does not make use of future
modalities. Let us denote by \LTL the set of pure future formulas, and by \LTLFP
the set of pure past formulas.

\LTLP is interpreted over \emph{state sequences}, which are finite or infinite
words over $2^\Sigma$. Given a state sequence $\sigma\in(2^\Sigma)^+$ or
$\sigma\in(2^\Sigma)^\omega$, the \emph{satisfaction} of a formula $\phi$ by
$\sigma$ at a time point $i\ge0$, denoted by $\sigma,i\models\phi$, is defined
as follows:

\begin{conditions}
  \item $\sigma,i \models p$                 & $p\in\state_i$;
  \item $\sigma,i \models \ltl{\neg\phi}$       & $\sigma,i \not\models \phi$;
  \item $\sigma,i \models \ltl{\phi_1 || \phi_2}$  &
          $\sigma,i \models \phi_1$ or $\sigma,i \models \phi_2$;
  \item $\sigma,i \models \ltl{\phi_1 && \phi_2}$ &
          $\sigma,i \models \phi_1$ and $\sigma,i \models \phi_2$;
  \item $\sigma,i \models \ltl{X\phi}$     & 
          $i+1<|\sigma|$ and  $\sigma,i+1\models \phi$;
  \item $\sigma,i \models \ltl{wX\phi}$     & 
          either $i+1=|\sigma|$ or $\sigma,i+1\models \phi$;
  \item $\sigma,i \models \ltl{Y\phi}$    &
          $i > 0$ and $\sigma,i-1\models \phi$;
  \item $\sigma,i \models \ltl{Z\phi}$    &
          either $i = 0$ or $\sigma,i-1\models \phi$;
  \item $\sigma,i \models \ltl{\phi_1 U \phi_2}$  &
          there exists $i\le j<|\sigma|$ such that $\sigma,j\models\phi_2$,
          \newline
          and $\sigma,k\models\phi_1$ for all $k$, with $i \le k < j$;
  \item $\sigma,i \models \ltl{\phi_1 S \phi_2}$    &
          there exists $j\le i$ such that $\sigma,j\models\phi_2$,\newline
          and $\sigma,k\models\phi_1$ for all $k$, with $j < k \le i$;
  \item $\sigma,i \models \ltl{\phi_1 R \phi_2}$  &
          either $\sigma,j\models\phi_2$ for all $i\le j < |\sigma|$, or there 
          exists \newline
          $k\ge i$ such that $\sigma,k\models\phi_1$ and\newline
          $\sigma,j\models\phi_2$ for all $i\le j \le k$;
  \item $\sigma,i \models \ltl{\phi_1 T \phi_2}$  &
          either $\sigma,j\models\phi_2$ for all $0\le j \leq i$, or there
          exists\newline $k \le i$ such that $\sigma,k\models\phi_1$ and 
          \newline
          $\sigma,j\models\phi_2$ for all $i\ge j \ge k$
\end{conditions}

Some connectives/operators of the language
can be defined in terms of a small number of basic ones. In particular,
$\phi_1\land\phi_2$ (conjunction) can be defined in terms of disjunction
as $\neg(\neg\phi_1\lor\neg\phi_2)$, $\ltl{\phi_1 R \phi_2}$ (the \emph{release} 
operator) in terms of the \emph{until} one as 
$\ltl{\lnot(\lnot\phi_1 U\lnot\phi_2)}$, and $\ltl{\phi_1 T
\phi_2}$ (the \emph{triggered} operator) in terms of 
the \emph{since} one as $\ltl{\lnot(\lnot\phi_1 S \lnot\phi_2)}$. 
Nevertheless, we consider
all these connectives and operators as primitive ones in order to be able to put any
formula in \emph{negated normal form} (NNF), that is, a form where negation is only
applied to proposition letters. 
Note that the syntax includes both a
\emph{tomorrow} ($\ltl{X\phi}$) and a \emph{weak tomorrow} ($\ltl{wX\phi}$)
operator, and, similarly, a \emph{yesterday} ($\ltl{Y\phi}$) and a \emph{weak
yesterday} ($\ltl{Z\phi}$) operator. 
Finally, standard shortcut operators are available such as the \emph{eventually}
($\ltl{F\phi}\equiv\ltl{\true U \phi}$) and \emph{always}
($\ltl{G\phi}\equiv\ltl{\neg F\neg\phi}$) future modalities, and the \emph{once}
($\ltl{O\phi}\equiv\ltl{\true S\phi}$) and \emph{historically}
($\ltl{H\phi}\equiv\ltl{\neg O\neg\phi}$) past modalities. 

We say that a state sequence $\sigma$ satisfies $\phi$, written
$\sigma\models\phi$, if $\sigma,0\models\phi$. 
If $\phi$ belongs to \LTLFP, that is, the pure past fragment of \LTLP, then we
interpret $\phi$ only on \emph{finite} state sequences and we say that
$\sigma \in (2^\Sigma)^+$ is a model of $\phi$ if and only if
$\sigma,|\sigma|-1 \models \phi$, \ie, each $\phi$ in \LTLFP is
interpreted at the \emph{last} state of a finite state sequence.

Notice that, when interpreted over an infinite word, the semantics of the \emph{tomorrow} and
\emph{weak tomorrow} operators is the same. The \emph{language}
of $\phi$, denoted by $\lang(\phi)$, is the set of words
$\sigma\in(2^\Sigma)^\omega$ such that $\sigma\models\phi$. The
\emph{language of finite words} of $\phi$, denoted by $\langfin(\phi)$, is
the set of finite words $\sigma\in(2^\Sigma)^+$ such that
$\sigma\models\phi$. Given a logic $\mathsf{L}$, 
we denote by $\sem{\mathsf{L}}$ the set of languages $\lang$ such that there is
a formula $\phi\in\mathsf{L}$ such that $\lang=\lang(\phi)$, and 
by $\semfin{\mathsf{L}}$ the set of languages of finite words $\lang$ such
that there is a formula $\phi\in\mathsf{L}$ such that
$\lang=\langfin(\phi)$ ($\semfin{\LTL}$ is usually referred to by \LTLf
in the literature~\cite{DeGiacomoV13}). It is known that \LTLf  and pure past \LTL
(\LTLFP) have the same expressive power \cite{lichtenstein1985glory,thomas1988safety}.

\begin{prop}
\label{prop:reverse}
  $\semfin{\LTL} = \semfin{\LTLFP}$.
\end{prop}

We now define the two fragments of \LTL that are the subject of this paper.
\begin{defi}[\safetyltl and \cosafetyltl \cite{sistla1994safety}]
\label{def:cosafetyltl}
  The logic \safetyltl (resp., the logic \cosafetyltl) is the fragment of \LTL where, 
  for formulas in \emph{negated normal form}, only the \emph{tomorrow}, \emph{weak
  tomorrow}, and \emph{release} (resp., \emph{until}) temporal modalities are
  allowed.
\end{defi}

Note that both \safetyltl and \cosafetyltl contain only \emph{future}
temporal operators.
We also define the logic $\cosafetyltlnoweak$ as the logic $\cosafetyltl$
devoid of the \emph{weak tomorrow} ($\ltl{wX}$) operator (this logic will
play a central role in our proofs). Similarly, we define $\safetyltlnonext$
as the logic \safetyltl devoid of the \emph{tomorrow} ($\ltl{X}$) operator.

In the next section, we introduce two fragments of the \emph{First-Order Theory of
Linear Orders} \cite{buchi1960weak,buchi1990decision}, namely \FOTLO (or simply 
$\FO$ for short). 
Given an alphabet $\Sigma$, $\FO$ is a first-order language with equality over the
signature $\seq{<,\set{P}_{p\in\Sigma}}$, and is interpreted over structures
$\mathcal{M} = \seq{D^{\mathcal{M}},<^{\mathcal{M}},\set{P^{\mathcal{M}}}_{p\in\Sigma}}$, where
$D^{\mathcal{M}}$ 
is either the set $\N$ of natural numbers or a
prefix $\set{0,\ldots,n}$ thereof, and $<^{\mathcal{M}}$ is the usual ordering
relation over natural numbers. A \emph{sentence} of \FO is a formula of \FO
with no free variables. Given an $\FO$ formula $\phi(x_0,\ldots,x_m)$,
with $m+1$ free variables, the satisfaction of $\phi$ by a first-order structure
$\mathcal{M}$ when $x_0=n_0,\ldots,x_m=n_m$, denoted by
$\mathcal{M},n_0,\ldots,n_m\models\phi(x_0,\ldots,x_m)$, is defined according
the standard first-order semantics. State sequences over $\Sigma$ map naturally
into such structures. Given a word $\sigma\in(2^\Sigma)^*$ or
$\sigma\in(2^\Sigma)^\omega$, we denote by $(\sigma)^s$ the corresponding
first-order structure. Given a formula $\phi(x)$ with exactly one free variable $x$,
the \emph{language} of $\phi$, denoted by $\lang(\phi)$, is the set of words
$\sigma\in(2^\Sigma)^\omega$ such that $(\sigma)^s,0\models\phi$. Similarly, the
\emph{language of finite words} of $\phi$, denoted by $\langfin(\phi)$, is the
set of finite words $\sigma\in(2^\Sigma)^+$ such that $(\sigma)^s\models\phi$.
We denote by $\sem{\FO}$ and $\semfin{\FO}$ the set of languages of respectively 
infinite and finite words definable by a \FO formula. 

Given a class of languages of finite words $\semfin{\mathsf{L}}$, we denote by 
$\semfininf{\mathsf{L}}$ the set of languages 
$\set{\lang \cdot (2^\Sigma)^\omega \suchthat \lang\in\semfin{\mathsf{L}}}$.
From now on, given a formula $\phi \in L$, we denote by $|\phi|$ the
number of symbols in $\phi$.

We conclude the section by recalling some fundamental known results.
\begin{prop}[%
  Kamp~\cite{kamp1968tense} and Gabbay~\cite{gabbay1980temporal}%
] \label{th:ltlfo}~\\
  $\sem{\LTL} = \sem{\FO}$ and $\semfin{\LTL} = \semfin{\FO}$.
\end{prop}

Finally, we state a normal form for \LTL-definable safety/co-safety
languages.

\begin{prop}[Chang \emph{et~al.}~\cite{ChangMP92}, Thomas~\cite{thomas1988safety}]
  \label{prop:galphafalpha}
  A language $\lang\in\sem{\LTL}$ is \emph{safety} (resp., \emph{co-safety}) if
  and only if it is the language of a formula of the form $\ltl{G\alpha}$
  (resp., $\ltl{F\alpha}$), where $\alpha\in\LTLFP$.
\end{prop}



\section{\SafetyFO and \coSafetyFO}
\label{sec:fragments}

In this section, we state and prove the main results of the paper: we define two simple fragments of \FO and
we show that they precisely capture \safetyltl and \cosafetyltl, respectively.
A summary of the achieved results 
is given in \cref{fig:summary}.


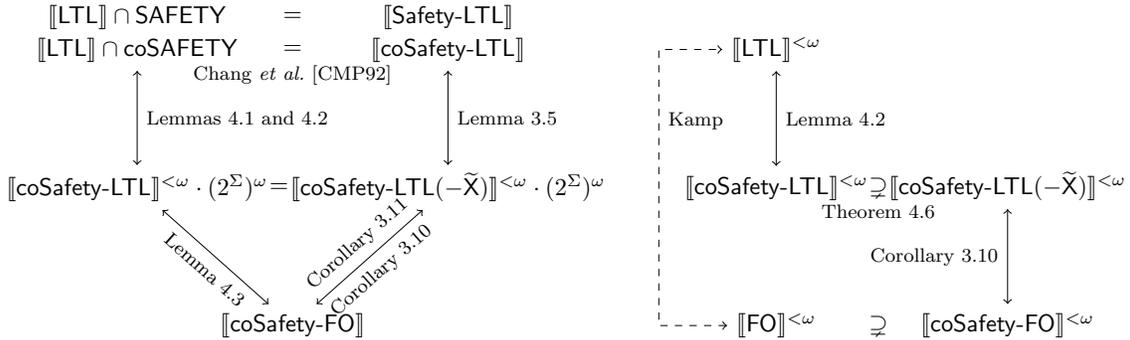
\begin{figure}
  \resizebox{\textwidth}{!}{
    \begin{tikzpicture}[
      sibling distance=4.5cm,
      level distance=2cm,
      edge label/.style={midway,font=\scriptsize},
      font=\small
    ]
      \node[math mode] (cosafety-fo) {\sem{\coSafetyFO}} [grow=up]
        child {
          node[math mode] (cosafetyltlnoweak)
            {\semfininf{\cosafetyltlnoweak}} edge from parent[draw=none]
          child {
            node[math mode] (cosafetyltl) {\sem{\cosafetyltl}} 
              edge from parent[draw=none]
          }
        } child {
          node[math mode] (cosafetyltlfininf)
            {\semfininf{\cosafetyltl}} edge from parent[draw=none]
          child {
            node[math mode] (cosafetylangs) 
              {\sem{\LTL}\cap\coSAFETY} edge from parent[draw=none]
          }
        };
  
      \path (cosafetylangs) node[above,yshift=1ex,math mode] (safetylangs) 
        {\sem{\LTL}\cap\SAFETY};
  
      \path (cosafetyltl) node[above,yshift=1ex,math mode] (safetyltl) 
        {\sem{\safetyltl}};
  
      \draw[<->] (cosafety-fo) -- (cosafetyltlfininf) 
        node[edge label,sloped,anchor=north] 
        {\cref{lem:cosafetyltlinfsonesfoguardinf}};
  
      \draw[<->] (cosafetyltlfininf) -- (cosafetylangs)
        node[edge label,anchor=west]
        {
          \cref{lemma:cosafetytoltl,lem:ltlfinequivcosafetyltlfin}
        };
  
      \draw[<->] (cosafety-fo) -- (cosafetyltlnoweak)
        node[edge label, sloped, anchor=south]
          {\cref{lemma:cosafetyfoseminf}}
        node[edge label, sloped, anchor=north] 
          {\cref{cor:cosafetyfoiscosafetyltlnoweak}};
  
      \draw[<->] (cosafetyltlnoweak) -- (cosafetyltl)
        node[edge label, anchor=west]
          {\cref{lem:cosafetyinfcosafetynoweakfin}};
  
      \path (cosafetyltlfininf) -- (cosafetyltlnoweak) node[midway] {$=$};
      
      \path (cosafety-fo |- cosafetyltl) node {$=$} 
        node[below = 0.1cm,font=\scriptsize] {
          Chang \emph{et~al.}~\cite{ChangMP92}
        };
  
      \path (cosafety-fo |- safetyltl) node {$=$};
  
      \path (cosafetyltlnoweak.east) ++(1cm,0)
        node[anchor=west,inner xsep=0pt] 
          (semfincosafetyltl) {$\semfin{\cosafetyltl}$};
      \path (semfincosafetyltl.east)
        node[anchor=west,inner xsep=0pt] (ne) {$\supsetneq$};
      \path (ne.east)
        node[anchor=west, inner xsep=0pt] 
          (semfincosafetyltlnoweak) {$\semfin{\cosafetyltlnoweak}$};
  
      \path (ne) node[below,yshift=-3pt, font=\scriptsize] { 
        \cref{thm:negativecosafetyinfinitefinite}
      };
  
      \path (semfincosafetyltl |- cosafetyltl)
        node (semfinltl) {$\semfin{\LTL}$};
  
      \path (semfincosafetyltl |- cosafety-fo)
        node (semfinfo) {$\semfin{\FO}$};
  
      \path(semfincosafetyltlnoweak |- semfinfo)
        node (semfincosafetyfo) {$\semfin{\coSafetyFO}$};
  
      \draw[<->] (semfinltl) -- (semfincosafetyltl)
        node[edge label, anchor=west] (semfinltlsemfinsafetyltllabel){
          \cref{lem:ltlfinequivcosafetyltlfin}
        };
  
      \draw[<->,dashed] (semfinfo) -- ++(-1.7,0) |- (semfinltl);

      \path (semfinltlsemfinsafetyltllabel.base west) ++(-1.7,0) 
        node[font=\scriptsize, anchor=base west] { Kamp };
  
      \draw[<->] (semfincosafetyfo) -- (semfincosafetyltlnoweak)
        node[edge label, anchor=east] {
          \cref{cor:cosafetyfoiscosafetyltlnoweak}
        };
  
      \path (ne |- semfincosafetyfo) node {$\supsetneq$};
      
    \end{tikzpicture}
  }
  \caption{
    Summary of the results 
    about languages over infinite words on the left and over finite words on the right. 
    Solid arrows are own results; dashed arrows are known from literature.
  }
  \label{fig:summary}
\end{figure}

\begin{defi}[\SafetyFO]
  The logic \SafetyFO is generated by the following grammar:
  \begin{align*}
    \mathit{atomic} & {} ::= x < y \choice x = y \choice x \ne y \choice P(x) 
                                     \choice \neg P(x) \\
    \phi & {}::= \mathit{atomic} 
                    \choice \phi\lor\phi 
                    \choice \phi\land\phi \choice 
                    \exists y(x < y < z \land \phi) \choice 
                    \forall y(x < y \implies \phi)
  \end{align*}
  where $x$, $y$, and $z$ are first-order variables, $P$ is a unary
  predicate, and $\phi_1$ and $\phi_2$ are \SafetyFO formulas.
\end{defi}

\begin{defi}[\coSafetyFO]
  The logic \coSafetyFO is generated by the following grammar:
  \begin{align*}
    \mathit{atomic} & {}::= x < y \choice x = y \choice x \ne y \choice P(x) 
                                     \choice \neg P(x) \\
    \phi & {}::= \mathit{atomic} 
                    \choice \phi\lor\phi 
                    \choice \phi\land\phi \choice 
                    \exists y(x < y \land \phi) \choice 
                    \forall y(x < y < z \implies \phi)
  \end{align*}
  where $x$, $y$, and $z$ are first-order variables, $P$ is a unary
  predicate, and $\phi_1$ and $\phi_2$ are \coSafetyFO formulas.
\end{defi}

We need to make a few observations on the syntax of the two fragments. First of
all, note how any formula of \SafetyFO is the negation of a formula of
\coSafetyFO and \viceversa, and how any formula in this fragments has at
least one free variable. 
Then, note that the two fragments are defined in
\emph{negated normal form}, \ie negation only appears on atomic formulas. The
particular kind of existential and universal quantifications allowed are the
culprit of these fragments. In particular, \SafetyFO restricts any existentially
quantified variable to be bounded between two free variables. The
same applies to universal quantification in \coSafetyFO. Moreover \SafetyFO and
\coSafetyFO formulas are \emph{future formulas}, \ie the quantifiers can only
range over values \emph{greater} than some free variables. These two
features are essential to precisely capture \safetyltl and \cosafetyltl.
Finally, note that the comparisons in the guards of the quantifiers are strict,
but non-strict comparisons can be used as well. In particular, $\exists y(x\le y
\land \phi)$ can be rewritten as $\phi[y/x]\lor \exists y(x < y \land \phi)$,
where $\phi[y/x]$ is the formula obtained by renaming all the free
occurrences of $y$ in $\phi$
with $x$. Similarly, $\forall z(x\le z \le y \implies \phi)$ can be rewritten as
$\phi[z/x]\land\phi[z/y]\land\forall z(x < z < y \implies \phi)$.

To prove the relationship between \safetyltl, \cosafetyltl, and these fragments,
we focus now on \coSafetyFO. By duality, all the results transfer to \SafetyFO.
We focus on \coSafetyFO because the unbounded quantification is existential, and
it is easier to reason about the existence of prefixes than on all the prefixes
at once.
We start by observing that, since the \emph{weak tomorrow} operator, over
infinite words, coincides with the \emph{tomorrow} operator, the following
holds.\fitpar
\begin{observation}
  \label{obs:observationnoweaknext}
  $\sem{\cosafetyltl}=\sem{\cosafetyltlnoweak}$
\end{observation}

When reasoning over finite words, the \emph{weak tomorrow} operator plays
a crucial role, since it can be used to recognize when we are at the last
position of a word. In fact, the formula $\sigma,i \models \ltl{wX \false}$ is
true if and only if $i = |\sigma|-1$, for any $\sigma \in (2^\Sigma)^*$.

Now, let us note that, thanks to the absence of the \emph{weak tomorrow}
operator, the \cosafetyltlnoweak logic is such that the concatenation of
any (finite or infinite) suffix to a finite model of a \cosafetyltlnoweak
formula results in a correct model of a formula.
\Cref{lem:cosafetyfincosafetynoweakfin,lem:cosafetyinfcosafetynoweakfin}
prove this result for finite and infinite suffixes, respectively.

\begin{lem}
\label{lem:cosafetyfincosafetynoweakfin}
    $\semfin{\cosafetyltlnoweak} = \semfinfin{\cosafetyltlnoweak}$
\end{lem}
\begin{proof}
  We have to prove that, for each formula $\phi \in \cosafetyltlnoweak$, it
  holds that:
  \begin{align*}
    \langfin(\phi) = \langfin(\phi) \cdot (2^\Sigma)^*
  \end{align*}
  We proceed by induction on the structure of $\phi$. 
  For the base case, consider $\ltl{\phi = p \in \Sigma}$. The case for
  $\ltl{\phi = ! p}$ is similar.  Let $\sigma \in (2^\Sigma)^*$.  It holds
  that $\sigma\in\langfin(p)$ iff $\sigma_0 \in \lang(p)$ iff $\sigma_0
  \cdot \sigma'\in \lang(p)$ (for any $\sigma' \in (2^\Sigma)^*$) iff
  $\sigma \in \langfin(p)\cdot(2^\Sigma)^*$.

  For the inductive step:
  \begin{enumerate}
  \item Let $\ltl{\phi = \phi_1 \land \phi_2}$. It holds that $\sigma \in
      \langfin(\phi_1 \land \phi_2)$ iff $\sigma \in \langfin(\phi_1)$ and
      $\sigma\in\langfin(\phi_2)$. By inductive hypothesis, this is
      equivalent to $\sigma\in\langfin(\phi_1)\cdot(2^\Sigma)^*$ and
      $\sigma\in\langfin(\phi_2)\cdot(2^\Sigma)^*$. This holds iff
      $\sigma\in \langfin(\phi_1\land\phi_2)\cdot(2^\Sigma)^*$.
  \item Let $\ltl{\phi = \phi_1 \lor \phi_2}$. It holds that $\sigma \in
      \langfin(\phi_1 \lor \phi_2)$ iff $\sigma \in \langfin(\phi_1)$ or
      $\sigma\in\langfin(\phi_2)$. By inductive hypothesis, this is
      equivalent to $\sigma\in\langfin(\phi_1)\cdot(2^\Sigma)^*$ or
      $\sigma\in\langfin(\phi_2)\cdot(2^\Sigma)^*$. This holds iff
      $\sigma\in \langfin(\phi_1\lor\phi_2)\cdot(2^\Sigma)^*$.
  \item Let $\ltl{\phi = X \phi_1}$. It holds that
      $\sigma\in\langfin(\ltl{X\phi_1})$ iff
      $\sigma_{[1,|\sigma|-1]}\in\langfin(\phi_1)$. By inductive
      hypothesis, this is equivalent to $\sigma_{[1,|\sigma|-1]} \in
      \langfin(\phi_1)\cdot(2^\Sigma)^*$. This holds iff $\sigma \in
      \langfin(\ltl{X\phi_1})\cdot(2^\Sigma)^*$.
  \item Let $\ltl{\phi = \phi_1 U \phi_2}$.  It holds that
      $\sigma_{[i,|\sigma|-1]} \in \langfin(\phi_1)$ and
      $\sigma_{[j,|\sigma|-1]} \in \langfin(\phi_2)$, for some $0\le
      i<|\sigma|$ and for all $0\le j< i$. By inductive hypothesis,
      $\sigma_{[i,|\sigma|-1]} \in \langfin(\phi_1)\cdot(2^\Sigma)^*$ and
      $\sigma_{[j,|\sigma|-1]} \in \langfin(\phi_2)\cdot(2^\Sigma)^*$ (for
      some $0\le i<|\sigma|$ and for all $0\le j< i$). This is equivalent
      to $\sigma\in\langfin(\ltl{\phi_1 U \phi_2})\cdot(2^\Sigma)^*$. \qedhere
  \end{enumerate}
\end{proof}

The following lemma generalizes \cref{lem:cosafetyfincosafetynoweakfin} to
the case of infinite words.

\begin{lem}
  \label{lem:cosafetyinfcosafetynoweakfin}
    $\sem{\cosafetyltlnoweak} = \semfininf{\cosafetyltlnoweak}$
\end{lem}
\begin{proof}
  We have to prove that, for each formula $\phi \in \cosafetyltlnoweak$, it
  holds that:
  \begin{align*}
    \lang(\phi) = \langfin(\phi) \cdot (2^\Sigma)^\omega
  \end{align*}
  We proceed by induction on the structure of $\phi$. 
  For the base case, consider $\ltl{\phi = p \in \Sigma}$. The case for
  $\ltl{\phi = ! p}$ is similar.  Let $\sigma \in (2^\Sigma)^\omega$.  It
  holds that $\sigma\in\lang(p)$ iff $\sigma_0 \in \lang(p)$ iff $\sigma_0
  \cdot \sigma'\in \lang(p)$ (for any $\sigma' \in (2^\Sigma)^\omega$) iff
  $\sigma \in \langfin(p)\cdot(2^\Sigma)^\omega$.

  For the inductive step:
  \begin{enumerate}
  \item Let $\ltl{\phi = \phi_1 \land \phi_2}$. It holds that $\sigma \in
    \lang(\phi)$ iff $\sigma \in \lang(\phi_1)$ and $\sigma \in
      \lang(\phi_2)$. By the inductive hypothesis, this is equivalent to
      $\sigma \in \langfin(\phi_1)\cdot(2^\Sigma)^\omega$ and $\sigma \in
      \langfin(\phi_2)\cdot(2^\Sigma)^\omega$. This means that there exist
      two indices $i,j \in \N$ such that $\sigma_{[0,i]} \in
      \langfin(\phi_1)$ and $\sigma_{[0,j]} \in \langfin(\phi_2)$. Let $m$
      be the greatest between $i$ and $j$. By
      \cref{lem:cosafetyfincosafetynoweakfin} it holds that $\sigma_{[0,m]}
      \in \langfin(\phi_1)$ and $\sigma_{[0,m]} \in\langfin(\phi_2)$, \ie
      $\sigma_{[0,m]} \in \langfin(\phi_1 \land \phi_2)$.  Therefore
      $\sigma \in \langfin(\phi_1 \land \phi_2) \cdot (2^\Sigma)^\omega$.
  \item Let $\ltl{\phi = \phi_1 \lor \phi_2}$. It holds that $\sigma \in
    \lang(\phi)$ iff $\sigma \in \lang(\phi_1)$ or $\sigma \in
      \lang(\phi_2)$. Without loss of generality, we consider the case that
      $\sigma \in \lang(\phi_1)$ (the other case is specular). By the
      inductive hypothesis, this is equivalent to $\sigma \in
      \langfin(\phi_1)\cdot(2^\Sigma)^\omega$. Therefore, $\sigma
      = \sigma'\cdot\sigma''$ where $\sigma'\in\langfin(\phi_1)$ and
      $\sigma''\in(2^\Sigma)^\omega$. Since $\langfin(\phi_1) \subseteq
      \langfin(\phi_1 \lor \phi_2)$, this is equivalent to $\sigma \in
      \langfin(\phi_1 \lor \phi_2) \cdot (2^\Sigma)^\omega$.
  \item Let $\ltl{\phi = X \phi_1}$. It holds that
      $\sigma\in\lang(\ltl{X\phi_1})$ iff
      $\sigma_{[1,\infty)}\in\lang(\phi_1)$. By inductive hypothesis, this
      is equivalent to $\sigma_{[1,\infty)} \in
      \langfin(\phi_1)\cdot(2^\Sigma)^\omega$. This holds iff $\sigma \in
      \langfin(\ltl{X\phi_1})\cdot(2^\Sigma)^\omega$.
  \item Let $\ltl{\phi = \phi_1 U \phi_2}$. By the semantics of the
    \emph{until} operator, it holds that $\sigma \in \lang(\phi)$ iff there
      exists an index $i \in \N$ such that $\sigma_{[i,\infty)} \in
      \lang(\phi_2)$ and $\sigma_{[j,\infty)} \in \lang(\phi_1)$ for all $0
      \le j < i$.  By the inductive hypothesis, this is equivalent to
      $\sigma_{[i,\infty)} \in \langfin(\phi_2)\cdot(2^\Sigma)^\omega$ and
      $\sigma_{[j,\infty)} \in \langfin(\phi_1)\cdot(2^\Sigma)^\omega$ for
      all $0 \le j < i$. This means that there exists an index $i \in \N$
      and $i+1$ indices $k_0, \dots , k_i \in \N$ such that
      $\sigma_{[i,k_i]} \in \langfin(\phi_2)$ and $\sigma_{[j,k_j]} \in
      \langfin(\phi_1)$ for all $0 \le j < i$. Let $m$ be the greatest
      between $k_0 , \dots , k_i$. By
      \cref{lem:cosafetyfincosafetynoweakfin}, it holds that there exists
      an index $i \in \N$ such that $\sigma_{[i,m]} \finmodels \phi_2$ and
      $\sigma_{[j,m]} \finmodels \phi_1$ for all $0 \le j < i$. Therefore,
      this is equivalent to $\sigma \in \langfin(\ltl{\phi_1
      U \phi_2})\cdot(2^\Sigma)^\omega$. \qedhere
  \end{enumerate}
\end{proof}

Note that an equality similar to
\cref{lem:cosafetyfincosafetynoweakfin,lem:cosafetyinfcosafetynoweakfin}
does not hold if we allow the \emph{weak tomorrow} operator, because if
$\sigma\in\langfin(\ltl{wX\phi})$, it might still be that $|\sigma|=1$ and
$\phi$ is unsatisfiable, hence there is no way to extend $\sigma$ to an
infinite word while still satisfying the formula.

In~\cite{DeGiacomoDM14}, De Giacomo \etal define the notion of
\emph{insensitive to infiniteness} as a way to compare the finite and the
infinite word semantics of fragments of \LTL. They define
a formula $\phi$ (over an alphabet $\Sigma$) to be \emph{insensitive to
infiniteness} if, and only if, for any finite word $\sigma\in\Sigma^+$, it
holds that $\sigma \models \phi$ iff $\sigma \cdot \set{\mathsf{e}^\omega}
\models \phi$, where $\mathsf{e}$ is a fresh proposition letter
($\mathsf{e} \not\in \Sigma$).  By
\cref{lem:cosafetyinfcosafetynoweakfin}, it follows that every formula of \cosafetyltlnoweak is
insensitive to infiniteness.

Then, we can focus on $\cosafetyltlnoweak$ and $\coSafetyFO$ on finite words. If
we can prove that $\semfin{\cosafetyltlnoweak}=\semfin{\coSafetyFO}$, we are
done. At first, we show how to encode $\cosafetyltlnoweak$ formulas into
$\coSafetyFO$ with exactly one free variable.
\begin{lem}
  \label{lemma:fromTLtocosafetyfo}
  $\semfin{\cosafetyltlnoweak}\subseteq\semfin{\coSafetyFO}$
\end{lem}
\begin{proof}
  Let $\lang \in \semfin{\cosafetyltlnoweak}$, and let
  $\phi\in\cosafetyltlnoweak$ such that $\lang = \langfin(\phi)$.
  By following the semantics of the operators in $\phi$, we can obtain
  an equivalent \coSafetyFO formula $\phi_{\FO}$.
  We inductively define the formula $FO(\phi,x)$, where $x$ is a variable, as 
  follows:\fitpar
  \begin{itemize}
    \item
      $FO(p,x) = P(x)$, for each $p \in \Sigma$
    \item
      $FO(\lnot p, x) = \lnot P(x)$, for each $p \in \Sigma$
    \item
      $FO(\phi_1 \land \phi_2,x) = FO(\phi_1,x) \land FO(\phi_2,x)$
    \item
      $FO(\phi_1 \lor \phi_2,x) = FO(\phi_1,x) \lor FO(\phi_2,x)$
    \item
      $FO(\ltl{X \phi_1},x) = \exists y(x < y \land y=x+1 \land FO(\phi_1,y))$\\
      where $y=x+1$ can be expressed as $\forall z(x<z<y \implies \false)$.
    \item 
      $FO(\ltl{\phi_1 U \phi_2},x) = {}
        \exists y( x \le y \land FO(\phi_2,y) \land 
          \forall z (x \le z < y \implies FO(\phi_1,z)))$
  \end{itemize}
  For each $\phi \in \cosafetyltlnoweak$, the formula $FO(\phi,x)$ has exactly
  one free variable $x$. 
  It is easy to see that for all finite state sequences $\sigma \in
  (2^\Sigma)^*$, it holds that $\sigma \finmodels \phi$ if and only if
  $(\sigma)^s,0 \finmodels FO(\phi,x)$, and $FO(\phi,x) \in \coSafetyFO$.
  Therefore, $\lang \in \semfin{\coSafetyFO}$.
\end{proof}

It is time to show the opposite direction, \ie that any \coSafetyFO formula can
be translated into a \cosafetyltlnoweak formula which is equivalent over finite
words. To prove this fact we adapt a proof of Kamp's theorem by
Rabinovich~\cite{DBLP:journals/corr/Rabinovich14}. Kamp's theorem is one of the
fundamental results about temporal logics, which states that \LTL
corresponds to \FO in terms of expressiveness. Here, we prove a similar result
in the context of co-safety languages. The proof goes by introducing a
\emph{normal form} for \FO formulas, and showing that 
\begin{enumerate*}[label=(\roman*)]
  \item
    any \coSafetyFO formula can be translated into such normal form and 
  \item
    any formula in normal form can be straightforwardly translated into
    a \cosafetyltlnoweak formula. 
\end{enumerate*}
We start by introducing such a normal form.

\begin{defi}[$\decomp$-formulas]
  An $\decomp$-formula $\phi(z_0,\ldots,z_m)$ with $m$ free variables is a
  formula of this form:
  \begin{align*}
    \phi(z_0,\ldots,z_m) \coloneqq
    \exists x_0 &\ldots \exists x_n \big( \\
      &x_0 < x_1 < \dots < x_n && \text{ordering constraints} \\
      {}\land{}&z_0=x_0\land\bigwedge_{k=1}^{m} (z_{k} = x_{i_k}) 
        && \text{binding constraints} \\
      {}\land{}&\bigwedge_{j=0}^{n}\alpha_j(x_j) 
        && \text{punctual constraints} \\
      {}\land{}&\bigwedge_{j=1}^{n} \forall y (
        x_{j-1} < y < x_{j} \to \beta_j(y))
    \big) && \text{interval constraints}
  \end{align*}
  where $i_k \in \set{0,\ldots,n}$ for each $0 \le k \le m$, and $\alpha_j$ and
  $\beta_j$, for each $1 \le j \le n$, are quantifier-free formulas with exactly
  one free variable.
\end{defi}

Some explanations are due. Each $\decomp$-formula states a number of
requirements for its free variables and for its quantified variables. Through
the binding constraints, the free variables are identified with a subset of the
quantified variables in order to uniformly state the punctual and interval
constraints, and the ordering constraints which sort all the variable in a total
order. Note that there is no relationship between $n$ and $m$: there might be
more quantified variables than free variables, or less. Note as well that the
binding constraint $z_0=x_0$ is always present, \ie at least one free variable
has to be the minimal element of the ordering. This ensures that
$\decomp$-formulas constrain only positions of the word that are greater
than the value of $x_0$.

We say that a formula of \coSafetyFO is in \emph{normal form} if and only if it
is a disjunction of $\decomp$-formulas.
To see how formulas in normal form make sense, let us immediately show how to
translate them into \cosafetyltlnoweak formulas.
\begin{lem}
  \label{lem:fromS1SFOguardtoTL}
  For any formula $\phi(z) \in \coSafetyFO$ in normal form, with a single free
  variable, there exists a formula $\psi\in\cosafetyltlnoweak$ such that
  $\langfin(\phi(z)) = \langfin(\psi)$.
\end{lem}
\begin{proof}
  We show how any $\decomp$-formula is equivalent to a
  \cosafetyltlnoweak-formula, over finite words.  Since each formula in normal
  form is a disjunction of $\decomp$-formulas, and since \cosafetyltlnoweak is
  closed under disjunction, this implies the proposition.
  Let $\phi(z)$ be a $\decomp$-formula with a single free variable.
  Having only one free variable, $\phi(z)$ is of the form:
  \begin{align*}
    \exists x_0 \dots \exists x_n \big(&
      x_0 < \dots < x_n 
      {}\land{}z = x_0 \\
      &{}\land{}\bigwedge_{j=0}^{n} \alpha_j(x_j)
      {}\land{}\bigwedge_{j=1}^{n} \forall y (
        x_{j-1} < y < x_{j} \to \beta_j(y))
    \big)
  \end{align*}
  Now, let $A_i$ be the temporal formulas corresponding to $\alpha_i$ and $B_i$
  be the ones corresponding to $\beta_i$. Recall that $\alpha_i$ and $\beta_i$
  are quantifier free with only one free variable, hence this correspondence is
  trivial. Since $z$ is the first time point of the ordering mandated by the
  formula, we only need future temporal operators to encode $\phi$ into a
  \cosafetyltlnoweak formula $\psi$ defined as follows:
  \begin{align*}
    \psi =
    A_0 \land \ltl{X} ( B_0 \ltl{U} (A_1 \land \ltl{X} ( B_1 \ltl{U} A_2 
    \land 
      \dots 
    \ltl{X}(B_{n-1} \ltl{U} A_{n}) \dots )
    ))
  \end{align*}
  It can be seen that $\sigma,k \models \psi$ if and only if
  $(\sigma)^s,k\models\phi(z)$, for each $\sigma \in (2^\Sigma)^+$ and each
  $k\ge0$.
  Thus, $\langfin(\phi(z)) = \langfin(\psi)$.
\end{proof}

Two differences between our $\decomp$-formulas and those used by
Rabinovich~\cite{DBLP:journals/corr/Rabinovich14} are crucial: first, we do not
have unbounded universal requirements, but all interval constraints use bounded
quantifications, hence we do not need the \emph{always} operator to encode them;
second, our $\decomp$-formulas are \emph{future} formulas, hence we only need
future operators to encode them.

We now show that any \coSafetyFO formula can be translated into normal form,
that is, into a \emph{disjunction} of $\decomp$-formulas. 
\begin{lem}
  \label{lem:froms1sfoguardtonormalform}
  Any \coSafetyFO formula is equivalent to a disjunction of $\decomp$-formulas.
\end{lem}
\begin{proof}
  Let $\phi$ be a \coSafetyFO formula. We proceed by structural induction on
  $\phi$. For the base case, for each atomic formula $\phi(z_0,z_1)$ we provide
  an equivalent $\decomp$-formula $\psi(z_0,z_1)$:
  \begin{enumerate}
    \item if $\phi = (z_0 < z_1)$ then $\psi := \exists x_0 \exists x_1 (z_0
    = x_0 \land z_1 = x_1 \land x_0 < x_1)$;
  \item if $\phi = (z_0 = z_1)$, then $\psi := \exists x_0 (z_0 = x_0 \land
    z_1 = x_0)$.
  \item if $\phi = (z_0 \ne z_1)$, we can note that $\phi \equiv z_0 < z_1 
    \lor z_1 < z_0$ and then apply Item~1;
  \item if $\phi = P(z_0)$ then we define $\psi := \exists x_0 (z_0
    = x_0 \land P(x_0))$. Similarly if $\phi = \lnot P(z_0)$.
  \end{enumerate}
  For the inductive step:
  \begin{enumerate}
    \item The case of a disjunction is trivial.
    \item If $\phi(z_0,\dots,z_k)$ is a conjunction, by the inductive hypothesis
      each conjunct is equivalent to a disjunction of $\decomp$-formulas. By
      distributing the conjunction over the disjunction we can reduce ourselves to
      the case of a conjunction
      $\psi_1(z_0,\ldots,z_k)\land\psi_2(z_0,\ldots,z_k)$ of two
      $\decomp$-formulas~\footnotemark. In this case we have that:
      \begin{align*}
        \psi_1 &\equiv\exists x_0\ldots\exists x_n\big(x_0<\dots<x_n\land 
          z_0=x_0\land \ldots
        \big)  \\
        \psi_2&\equiv\exists x_{n+1}\ldots\exists x_m(x_{n+1}<\dots<x_m\land
          z_0=x_{n+1} \land \ldots
        )
      \end{align*}
      \footnotetext{%
        Note that, without loss of generality, we can assume that $\psi_1$
        and $\psi_2$ have the same free variables $z_1,\dots,z_k$. In the
        case one of the two is not using a variable (say $z_i$), then its
        binding constraint will not bind any variable to $z_i$.
      }
      Since the set of quantified variables in $\psi_1$ is disjoint from the set
      of quantified variables in $\psi_2$, we can distribute the existential
      quantifiers over the conjunction $\psi_1 \land \psi_2$, obtaining:
      \begin{align*}
        \psi_1\land\psi_2&\equiv{}\exists x_0 \ldots\exists x_n\exists x_{n+1}
        \ldots\exists x_m \\
        &\big(x_0<\dots<x_n\land x_{n+1}<\dots<x_m \land 
           z_0=x_0 \land z_0=x_{n+1} \land \ldots
        \big)
        \shortintertext{Note that we can identify $x_0$ and $x_{n+1}$, obtaining:}
        \psi_1\land\psi_2&\equiv{}\exists x_0\ldots\exists x_n\exists x_{n+2},\ldots\exists x_m \\
        &\big(x_0<\dots<x_n\land x_0<x_{n+2}<\dots<x_m \land {}\\
        &  
           z_0=x_0 \land
           \bigwedge_{i=1}^k(z_i=x_{j_i}) \land {}
           \bigwedge_{\substack{i=0 , i \neq n+1}}^m\alpha_i(x_i) \land \\
           &\bigwedge_{\substack{i=1 , i \neq n+1 \\ i \neq n+2}}^m\forall y(x_{i-1}<y<x_i\to\beta_i(y)) \land
           \forall y (x_0 < y <x_{n+2} \to \beta_{n+2})
        \big)
      \end{align*}
      where $j_i \in \set{0,\dots,k}$, for each $0\le i \le k$.
      Now, to turn this formula into a disjunction of $\decomp$-formulas, we
      consider all the possible interleavings of the variables that respect the
      two imposed orderings and explode the formula into a disjunction that
      consider each such interleaving. Let
      $X=\set{x_0,\ldots,x_n,x_{n+2},\ldots,x_m}$ and let $\Pi$ be the set of
      all the permutations of $X$ compatible with the orderings $x_0<\dots<x_n$
      and $x_0<x_{n+1}<\dots < x_m$. For any $\pi \in \Pi$, $\pi(0)=x_0$. Now,
      $\psi_1\land\psi_2$ becomes the disjunction of a set of $\decomp$-formulas
      $\psi_\pi$, for each $\pi\in\Pi$, defined as:
      \begin{align*}
        \psi_\pi\equiv{}&\exists x_{\pi(0)}\ldots\exists x_{\pi(m)}\\
        & \big(x_{\pi(0)}<\dots<x_{\pi(m)} \land {}\\
        & z_0=x_0 \land 
          \bigwedge_{i=1}^k(z_i=x_{\pi(j_i)}) \land 
          \bigwedge_{i=0}^m\alpha_i(x_i) \land \\
          &\bigwedge_{i=0}^m
            \forall y(x_{\pi(i-1)} < y < x_{\pi(i)} \to \beta^*_i(y))
        \big)
      \end{align*}
      where $\beta^*_i$ suitably combines the formulas $\beta$ according to the
      interleaving of the orderings of the original variables, and is defined as
      follows:
      \begin{equation*}
        \beta^*_i=\begin{cases}
          \beta_{\pi(i)} & 
            \text{if both $\pi(i),\pi(i-1)\le n$ or both $\pi(i),\pi(i-1)>n$ }\\
          \beta_{\pi(i)} \land \beta_{\pi(i-1)} &
            \text{if $\pi(i)\le n$ and $\pi(i-1)>n$ or \viceversa}
        \end{cases}
      \end{equation*}
      Then we have that $\psi_1\land\psi_2\equiv\bigvee_{\pi\in\Pi}(\psi_\pi)$,
      which is a disjunction of $\decomp$-formulas.
    \item Let $\phi(z_0,\dots,z_m) = \exists z_{m+1} \suchdot (z_i < z_{m+1} 
      \land
      \phi_1(z_0,\dots,z_m,z_{m+1}))$, for some $0 \le i \le m$.
      By the inductive hypothesis, this is equivalent to the formula $\exists
      z_{m+1} (z_i < z_{m+1} \land \bigvee_{k=0}^{j}
      \psi_k(z_0,\dots,z_m,z_{m+1}))$, where $\psi_k(z_0,\dots,z_m,z_{m+1})$ is
      a $\decomp$-formula, for each $0 \le k \le j$, that is:
      \begin{align*}
        \exists z_{m+1} \suchdot ( z_i < z_{m+1} \land
          \bigvee_{k=0}^{j} (
            \exists x_0 \dots &\exists x_{n_k} 
            \psi^\prime_k(z_0,\dots,z_{m+1},x_0,\dots,x_{n_k})
          )
        )
      \end{align*}
      %
      %
      By distributing the conjunction over the disjunction, we obtain:
      \begin{align*}
        \exists z_{m+1} \suchdot (
          \bigvee_{k=0}^{j} (
            (z_i < z_{m+1}) \land
            \exists x_0 \dots &\exists x_{n_k} 
            \psi^\prime_k(z_0,\dots,z_{m+1},x_0,\dots,x_{n_k})
          )
        )
      \end{align*}
      and by distributing the existential quantifier over the disjunction, we
      have:
      \begin{align*}
        \bigvee_{k=0}^{j} (
          \exists z_{m+1} (
            (z_i < z_{m+1}) \land
            \exists x_0 \dots &\exists x_{n_k} 
            \psi^\prime_k(z_0,\dots,z_{m+1},x_0,\dots,x_{n_k})
          )
        )
      \end{align*}
      Since the subformula $z_i < z_{m+1}$ does not contain the variables
      $x_0,\dots,x_n$, we can push it inside the existential quantification,
      obtaining:
      \begin{align*}
        \bigvee_{k=0}^{j} (
          \exists z_{m+1} \suchdot 
            \exists x_0 \dots &\exists x_{n_k} \suchdot 
            (
              (z_i < z_{m+1}) \land
              \psi^\prime_k(z_0,\dots,z_{m+1},x_0,\dots,x_{n_k})
            )
        )
      \end{align*}
      Now we divide in cases:
      \begin{enumerate}
        \item 
          suppose that the formula
          $\psi^\prime_k(z_0,\dots,z_{m+1},x_0,\dots,x_{n_k})$ contains the following
          conjuncts: $z_i = x_{l_i}$ and $z_{m+1} = x_{l_{m+1}}$, with $l_{i}
          = l_{m+1}$.
          It holds that these formulas are in contradiction with the formula $z_i
          < z_{m+1}$, that is:
          \begin{align*}
            (z_i < z_{m+1}) \land (z_i = x_{l_i}) \land (z_{m+1} = x_{l_{m+1}}) \equiv \false
          \end{align*}
          Therefore, the disjunct $(z_i < z_{m+1}) \land
          \psi^\prime_k(z_0,\dots,z_{m+1},x_0,\dots,x_{n_k})$ is equivalent to
          $\false$, and thus can be safely removed from the disjunction.
        \item
          suppose that the formula
          $\psi^\prime_k(z_0,\dots,z_{m+1},x_0,\dots,x_{n_k})$ contains the following
          conjuncts: $z_i = x_{l_i}$, $z_{m+1} = x_{l_{m+1}}$ (with $l_{i} \neq
          l_{m+1}$), and $x_{l_{m+1}} < \dots < x_{l_{i}}$.
          As in the previous case, it holds that:
          \begin{align*}
            (z_i < z_{m+1}) \land (z_i = x_{l_i}) \land (z_{m+1} = x_{l_{m+1}})
            \land (x_{l_{m+1}} < \dots < x_{l_{i}}) \equiv \false
          \end{align*}
          Thus, also in this case, this disjunct can be safely removed from the
          disjunction.
        \item
          otherwise, it holds that 
          the formula
          $\psi^\prime_k(z_0,\dots,z_{m+1},x_0,\dots,x_{n_k})$ contains the following
          conjuncts: $z_i = x_{l_i}$, $z_{m+1} = x_{l_{m+1}}$ (with $l_{i} \neq
          l_{m+1}$), and $x_{l_i} < \dots < x_{l_{m+1}}$.
          Therefore, the subformula $z_i < z_{m+1}$ is redundant, and can be safely
          removed from $\psi^\prime_k(z_0,\dots,z_{m+1},x_0,\dots,x_{n_k})$. The
          resulting formula is a $\decomp$-formula.
      \end{enumerate}
      After the previous transformation, we obtain:
      \begin{align*}
        \bigvee_{k=0}^{j^\prime} (
          \exists z_{m+1} \suchdot 
            \exists x_0 \dots &\exists x_{n_k} \suchdot 
              \psi^{\prime\prime}_k(z_0,\dots,z_{m+1},x_0,\dots,x_{n_k})
        )
      \end{align*}
      Finally, since each formula
      $\psi^{\prime\prime}_k(z_0,\dots,z_{m+1},x_0,\dots,x_{n_k})$ contains the
      conjunct $z_{m+1} = x_{l_{m+1}}$, we can safely remove the quantifier
      $\exists z_{m+1}$. We obtain the formula:
      \begin{align*}
        \bigvee_{k=0}^{j^\prime} (
            \exists x_0 \dots &\exists x_{n_k} \suchdot 
              \psi^{\prime\prime}_k(z_0,\dots,z_m,x_0,\dots,x_{n_k})
        )
      \end{align*}
      which is a disjunction of $\decomp$-formulas.
    \item Let $\phi(z_0,\dots,z_m) = \forall z_{m+1} (z_i < z_{m+1} < z_j \to
    \phi_1(z_0,\dots,z_m,z_{m+1}))$, for some $0 \le i,j \le m$. By the
    induction hypothesis we know that $\phi_1$ is equivalent to a disjunction
    $\bigvee_k \psi_k$ where $\psi_k$ are $\decomp$-formulas, \ie each $\psi_k$
    is of the form:
    \begin{align*}
      \psi_k\equiv \exists x_0,\ldots,x_n\big(&
        x_0<\ldots<x_n \land 
        z_0=x_0 \land
        \bigwedge_{l=1}^{m+1}(z_l=x_{u_l}) \land {} \\
        &\bigwedge_{l=0}^n\alpha_l(x_l) \land
        \bigwedge_{l=1}^n\forall y(x_{l-1}<y<x_l \to \beta_l(y))
      \big)
    \end{align*}
    Without loss of generality, we can suppose that $z_i$, $z_{m+1}$ and
    $z_j$ are binded to some variables $x_{u_i}$, $x_{u_{m+1}}$ and
    $x_{u_{j}}$ that are ordered consecutively, \ie $x_{u_i} < x_{u_{m+1}}
    < x_{u_j}$ with no other variable in between. That is because otherwise
    the ordering constraints and the binding constraints would be in
    conflict with the guard $z_i < z_{m+1} < z_j$ of the universal
    quantification, and the disjunct $\psi_k$ could be removed from the
    disjunction. 
    As a matter of fact, take for example a disjunct of $\bigvee_{k}\psi_k$
    with ordering constraints inducing the order $z_i < z_h < z_{m+1}$, for
    some $h$. The existence of such a $z_h$ is not guaranteed \emph{for
    each} value of $z_{m+1}$ between $z_i$ and $z_j$ because when
    $z_{m+1}=z_i+1$ there is no value between $z_i$ and $z_i+1$ (we are on
    discrete time models), and thus such a disjunct can be safely removed
    from $\bigvee_{k}\psi_k$.
    That said, we can now isolate all the
    parts of $\psi_k$ that talk about $z_{m+1}$,
    bringing them out of the existential quantification, obtaining $\psi_k\equiv \theta_k\land \eta_k$, where:
    \begin{align*}
      \theta_k\equiv {}& 
      z_i<z_{m+1}<z_j\\
      {}\land{} &\alpha(z_{m+1})\land\forall y(z_i<y<z_{m+1}\to\beta(y))\land
                          \forall y(z_{m+1}<y<z_i\to\beta'(y)) 
    \end{align*}
    \begin{align*}
      \eta_k\equiv{} \exists x_0,\ldots,x_n 
      \big(&x_0<\ldots<x_n \land
        z_0=x_0\land
        \bigwedge_{l=1}^m(z_l=x_{u_l}) \land \\
       & \smashoperator{\bigwedge_{\substack{l=0\\l\ne u_{m+1}}}^n}
          \alpha_l(x_l) \land
        \smashoperator{
          \bigwedge_{\substack{l=1\\l-1\ne u_i\\l\ne u_j\\}}^n
        }
          \forall y(x_{l-1}<y<x_l \to \beta_l(y))
      \big)
    \end{align*}
    Now, we have $\phi\equiv\forall z_{m+1}(z_i<z_{m+1}<z_j\to
    \bigvee_k(\theta_k\land\eta_k))$. We can distribute the head of the
    implication over the disjunction: 
    \begin{equation*}
      \phi\equiv\forall z_{m+1}(\bigvee_k( z_i<z_{m+1}<z_j\to
      (\theta_k\land\eta_k)))
    \end{equation*}
    and then over the conjunction, obtaining:
    \begin{equation*}
      \phi\equiv\forall z_{m+1}\big(\bigvee_k
      ((z_i<z_{m+1}<z_j\to\theta_k)\land(z_i<z_{m+1}<z_j\to\eta_k))\big)
    \end{equation*}
    In order to simplify the exposition, we now show how to proceed in the case
    of two disjuncts, which is easily generalizable. So suppose we have:
    \begin{equation*}
      \phi\equiv \forall z_{m+1}\left(\lor
        \begin{aligned}
        (z_i<z_{m+1}<z_j\to\theta_1)&\land(z_i<z_{m+1}<z_j\to\eta_1)\\
        (z_i<z_{m+1}<z_j\to\theta_2)&\land(z_i<z_{m+1}<z_j\to\eta_2)
        \end{aligned}
      \right)
    \end{equation*}
    We can a) distribute the disjunction over the conjunction (\ie convert
    in conjunctive normal form in the case of multiple disjuncts):
    \begin{equation*}
      \phi\equiv\forall z_{m+1}\left(
    \begin{aligned}
      &((z_i<z_{m+1}<z_j\implies\theta_1)\lor(z_i<z_{m+1}<z_j\implies\theta_2))\\
      {}\land{} &((z_i<z_{m+1}<z_j\implies\theta_1)\lor(z_i<z_{m+1}<z_j\implies\eta_2))\\
      {}\land{} &((z_i<z_{m+1}<z_j\implies\eta_1)\lor(z_i<z_{m+1}<z_j\implies\theta_2))\\
      {}\land{}
      &((z_i<z_{m+1}<z_j\implies\eta_1)\lor(z_i<z_{m+1}<z_j\implies\eta_2))\\
    \end{aligned}
    \right)
    \end{equation*}
    b) factor out the head of the implications:
    \begin{equation*}
      \phi\equiv\forall z_{m+1}\left(
    \begin{aligned}
                &(z_i<z_{m+1}<z_j\implies\theta_1\lor\theta_2)\\
      {}\land{} &(z_i<z_{m+1}<z_j\implies\theta_1\lor\eta_2)\\
      {}\land{} &(z_i<z_{m+1}<z_j\implies\eta_1\lor\theta_2)\\
      {}\land{} &(z_i<z_{m+1}<z_j\implies\eta_1\lor\eta_2)\\
    \end{aligned}
    \right)
    \end{equation*}
    and c) distribute the universal quantification over the conjunction,
    obtaining:
    \begin{equation*}
      \phi\equiv\left(
    \begin{aligned}
                &\forall z_{m+1}(z_i<z_{m+1}<z_j\implies\theta_1\lor\theta_2)\\
      {}\land{} &\forall z_{m+1}(z_i<z_{m+1}<z_j\implies\theta_1\lor\eta_2)\\
      {}\land{} &\forall z_{m+1}(z_i<z_{m+1}<z_j\implies\eta_1\lor\theta_2)\\
      {}\land{} &\forall z_{m+1}(z_i<z_{m+1}<z_j\implies\eta_1\lor\eta_2)\\
    \end{aligned}
    \right)
    \end{equation*}
    Now, note that $\eta_1$ and $\eta_2$ do not contain $z_{m+1}$ as a free
    variable, because we factored out all the parts mentioning $z_{m+1}$ into
    $\theta_1$ and $\theta_2$ before. Therefore we can push them out from the
    universal quantifications, obtaining:
    \begin{equation*}
      \phi\equiv\left(
    \begin{aligned}
                &\forall z_{m+1}(z_i<z_{m+1}<z_j\implies\theta_1\lor\theta_2)\\
      {}\land{} &\forall z_{m+1}(z_i<z_{m+1}<z_j\implies\theta_1)\lor\eta_2\\
      {}\land{} &\forall z_{m+1}(z_i<z_{m+1}<z_j\implies\theta_2)\lor\eta_1\\
      {}\land{} &\neg\exists z_{m+1}(z_i<z_{m+1}<z_j)\lor\eta_1\lor\eta_2\\
    \end{aligned}
    \right)
    \end{equation*}
    Now, note that $\neg\exists z_{m+1}(z_i<z_{m+1}<z_j)$ is equivalent to
    $z_i=z_j\lor z_j=z_i+1$, which is the disjunction of two formulas that can
    be turned into $\decomp$-formulas. Since both $\eta_1$ and $\eta_2$ are
    already $\decomp$-formulas and since we already know how to deal with
    conjunctions and disjunctions of $\decomp$-formulas, it remains to show that
    the universal quantifications in the formula above can be turned into
    $\decomp$-formulas. Take $\forall z_{m+1}(z_i<z_{m+1}<z_j\to\theta_1)$, \ie:
    \begin{equation*}
      \forall z_{m+1}\left(z_i<z_{m+1}<z_j\to
        \begin{aligned}
        &z_i<z_{m+1}<z_j \\
        {}\land{}&\alpha(z_{m+1})\\
        {}\land{}&\forall y(z_i<y<z_{m+1}\to\beta(y))\\
        {}\land{}&\forall y(z_{m+1}<y<z_j\to\beta'(y))
        \end{aligned}
      \right)
    \end{equation*}
        
    Note that the first conjunct of the consequent can be removed, since it is
    redundant. Now, this formula is requesting $\beta(y)$ for all $y$ between
    $z_i$ and $z_{m+1}$, but with $z_{m+1}$ that ranges between $z_i$ and
    $z_j-1$, hence effectively requesting $\beta(y)$ to hold between $z_i$ and
    $z_j$. Similarly for $\beta'(y)$, which has to hold for all $y$ between
    $z_i+1$ and $z_j$. 
    
    Hence, it is equivalent to: 
    \begin{align*}
      & z_i=z_j \\
      {}\lor{} & z_j = z_i+1 \\
      {}\lor{} & \exists x_{i+1}(z_i<x_{i+1} \land x_{i+1}=z_i+1 \land z_j=x_{i+1}+1 \land \alpha(x_{i+1}))\\
      {}\lor{}&\exists x_i \exists x_{i+1} \exists x_{j-1} \exists x_{j}
      \left(
        \begin{aligned}
          &x_i < x_{i+1} < x_{j-1} < x_j \\
          {}\land{}&z_i = x_i \land z_j = x_j \\
          {}\land{}&\alpha(x_{i+1})\land\alpha(x_{j-1})\\
          {}\land{}&\forall y(x_i < y < x_{i+1} \to \false) \\
          {}\land{}&\forall y(x_{j-1} < y < x_{j} \to \false) \\
          {}\land{}&\forall y(x_{i} < y < x_{j-1} \to \alpha(y)\land\beta(y)) \\
          {}\land{}&\forall y(x_{i+1} < y < x_{j} \to \alpha(y)\land 
          \beta^\prime(y)) 
        \end{aligned}
      \right) 
    \end{align*}
    which is a disjunction of a $\decomp$-formula and others that can be turned
    into disjunctions of $\decomp$-formulas. The reasoning is at all similar for
    $\forall z_{m+1}(z_i<z_{m+1}<z_j\implies\theta_1\lor\theta_2)$.\qedhere
  \end{enumerate}
\end{proof}
Any \coSafetyFO formula can be translated into a disjunction of
$\decomp$-formulas by\linebreak[4]\cref{lem:froms1sfoguardtonormalform}, and then to
a \cosafetyltlnoweak formula by \cref{lem:fromS1SFOguardtoTL}. Together with
\cref{lemma:fromTLtocosafetyfo}, we obtain the following.

\begin{cor}
  \label{cor:cosafetyfoiscosafetyltlnoweak}
  $\semfin{\coSafetyFO}=\semfin{\cosafetyltlnoweak}$
\end{cor}

Moreover,
\cref{cor:cosafetyfoiscosafetyltlnoweak,lem:cosafetyfincosafetynoweakfin,lem:cosafetyinfcosafetynoweakfin},
imply the following corollary.

\begin{cor}
\label{lemma:cosafetyfoseminf}
  It holds that:
  \begin{itemize}
    \item $\semfin{\coSafetyFO}=\semfinfin{\coSafetyFO}$
    \item $\sem{\coSafetyFO}=\semfininf{\coSafetyFO}$
  \end{itemize}
\end{cor}

We are now ready to state the main result of this section.
\begin{thm}
  \label{thm:cosafetyltltofo}
  $\sem{\cosafetyltl}=\sem{\coSafetyFO}$
\end{thm}
\begin{proof}
  We know that $\sem{\cosafetyltl}=\semfininf{\cosafetyltlnoweak}$ by
  Observation~\ref{obs:observationnoweaknext} and
  \cref{lem:cosafetyinfcosafetynoweakfin}. Since
  $\semfin{\cosafetyltlnoweak}=\semfin{\coSafetyFO}$ by
  \cref{cor:cosafetyfoiscosafetyltlnoweak}, we have that
  $\semfininf{\cosafetyltlnoweak}=\semfininf{\coSafetyFO}$. Then, by
  \cref{lemma:cosafetyfoseminf} we have that
  $\semfininf{\coSafetyFO}=\sem{\coSafetyFO}$, hence
  $\sem{\cosafetyltl}=\sem{\coSafetyFO}$.
\end{proof}
\begin{cor}
  $\sem{\safetyltl}=\sem{\SafetyFO}$
\end{cor}


\section{\SafetyFO captures \LTL-definable safety languages}
\label{sec:proof}

In this section, we prove that \coSafetyFO captures \LTL-definable co-safety
languages. By duality, we have that \SafetyFO captures \LTL-definable
safety languages, and by the equivalence shown in the previous section, this
provides a novel proof of the fact that \safetyltl captures \LTL-definable
safety languages.  We start by characterizing co-safety languages in terms of
\LTL over finite words.

\begin{lem}
  \label{lemma:cosafetytoltl}
  $\sem{\LTL}\cap\coSAFETY=\semfininf{\LTL}$
\end{lem}
\begin{proof}
  ($\subseteq$) By \cref{prop:galphafalpha} we know that each language
  $\lang\in\sem{\LTL}\cap\coSAFETY$ is definable by a formula of the form
  $\ltl{F\alpha}$ where $\alpha\in\LTLFP$. Hence for each $\sigma\in\lang$
  there exists an $n$ such that $\sigma,n\models\alpha$, hence
  $\sigma_{[0,n]},n\models\alpha$. Note that $\sigma_{[n+1,\infty]}$ is
  unconstrained. By replacing all the \emph{since}/\emph{yesterday}/\emph{weak
  yesterday} operators in $\alpha$ with \emph{until}/\emph{tomorrow}/\emph{weak
  tomorrow} operators, we obtain an \LTL formula $\alpha^r$ such that
  $(\sigma_{[0,n]})^r,0\models\alpha^r$ (where $\sigma^r$ is the reverse of
  $\sigma$). Since \LTL captures star-free
  languages~\cite{lichtenstein1985glory} and star-free languages are closed by
  reversal, there is also an \LTL formula $\beta$ such that
  $\sigma_{[0,n]},0\models\beta$. Hence
  $\lang=\langfin(\beta)\cdot(2^\Sigma)^\omega$, and we proved that
  $\sem{\LTL}\cap\coSAFETY\subseteq\semfininf{\LTL}$.

  ($\supseteq$) Given $\lang\in\semfininf{\LTL}$, we know
  $\lang=\langfin(\beta)\cdot(2^\Sigma)^\omega$ for some \LTL formula $\beta$.
  Hence, for each $\sigma\in\lang$ there is an $n$ such that
  $\sigma_{[0,n]},0\models\beta$. Since \LTL captures star-free languages and
  star-free languages are closed by reversal, there is an \LTL formula
  $\alpha^r$ such that $(\sigma_{[0,n]})^r,0\models\alpha^r$. Now, by replacing
  all the \emph{until}/\emph{tomorrow}/\emph{weak tomorrow} operators in
  $\alpha^r$ with \emph{since}/\emph{yesterday}/\emph{weak yesterday} operators,
  we obtain an \LTLFP formula $\alpha$ such that
  $\sigma_{[0,n]},n\models\alpha$. Hence, $\sigma$ is such that there is an $n$
  such that $\sigma,n\models\alpha$, \ie $\sigma\models\ltl{F\alpha}$.
  Therefore, by \cref{prop:galphafalpha}, $\lang\in\sem{\LTL}\cap\coSAFETY$, and
  this in turn implies that $\semfininf{\LTL}\subseteq\sem{\LTL}\cap\coSAFETY$.
\end{proof}
Now, we show that, over finite words, the \emph{release} and the
\emph{globally} modalities can be defined only in terms of the \emph{weak
tomorrow}, the \emph{until} and the \emph{eventually} modalities.
Similarly, we also show that, over finite trace, the \emph{until} and the
\emph{eventually} operators can be defined only in terms of the \emph{tomorrow},
the \emph{release} and the \emph{globally} modalities.
\begin{lem}
  \label{lem:ltlfinequivcosafetyltlfin}
  $\semfin{\LTL}=\semfin{\safetyltl}=\semfin{\cosafetyltl}$
\end{lem}
\begin{proof}
  Since \safetyltl and \cosafetyltl are fragments of \LTL, we only need to
  show one direction, \ie that $\semfin{\LTL}\subseteq\semfin{\safetyltl}$
  and $\semfin{\LTL}\subseteq\semfin{\cosafetyltl}$. At first, we show
  the case of \cosafetyltl. For each \LTL formula
  $\phi$, we can build a \cosafetyltl formula whose language over finite
  words is exactly $\langfin(\phi)$. The \emph{globally} operator can be
  replaced by means of an \emph{until} operator whose existential part
  always refers to the last position of the word. In turn, this can be done
  with the formula $\ltl{wX\false}$, which is true only at the final
  position:
  \begin{align*}
    \ltl{G\phi \equiv \phi U (\phi \land wX \false)}
  \end{align*}
  Similarly, the \emph{release} operator can be expressed by means of a
  \emph{globally} operator in disjunction with an \emph{until} operator:
  \begin{align*}
    \ltl{\phi_1 R \phi_2} &\equiv \ltl{G\phi_2 \lor (\phi_2 U (\phi_1
      \land \phi_2))} 
    \equiv \ltl{\big(\phi_2 U (\phi_2 \land wX \false)\big) \lor
      \big(\phi_2 U (\phi_1 \land \phi_2)\big)}
  \end{align*}
  Hence, $\semfin{\LTL}=\semfin{\cosafetyltl}$. Now, if we exploit the duality
  between the \emph{eventually}/\emph{until} and the
  \emph{globally}/\emph{release} operators, we obtain:
  \begin{align*}
    \ltl{F\phi} &\equiv \ltl{\phi R (\phi | X\true)} \\
    \ltl{\phi_1 U \phi_2} &\equiv 
      \ltl{\phi_2 R (\phi_2 | X\true) & \phi_2 R (\phi_1 \lor \phi_2)}
  \end{align*}
  Hence, since we showed that any \emph{eventually} operator and any
  \emph{until} operator can be defined only in terms of the
  \emph{tomorrow}, the \emph{globally}, and the \emph{release} operators, we
  have that $\semfin{\LTL}=\semfin{\safetyltl}$.
\end{proof}
Then, we relate \cosafetyltl on finite words and \coSafetyFO.

\begin{lem}
  \label{lem:cosafetyltlinfsonesfoguardinf}
  $\semfininf{\cosafetyltl}=\sem{\coSafetyFO}$
\end{lem}
\begin{proof}
  ($\subseteq$) We have that $\semfin{\cosafetyltl}=\semfin{\LTL}$ by
  \cref{lem:ltlfinequivcosafetyltlfin}, and this implies that
  $\semfininf{\cosafetyltl}=\semfininf{\LTL}$, and
  $\semfininf{\cosafetyltl}=\semfininf{\FO}$ by \cref{th:ltlfo}. Now, let
  $\phi\in\FO$, and suppose \emph{w.l.o.g.}\ that $\phi$ is in \emph{negated
  normal form}.
  We define the formula $\phi'(x,y)$, where $x$ and $y$ are two fresh variables
  that do not occur in $\phi$, as the formula obtained from $\phi$ by a)
  replacing each subformula of $\phi$ of type $\exists z \phi_1$ with $\exists
  z(x \le z \land \phi_1)$, and b) by replacing each subformula of $\phi$ of
  type $\forall z \phi_1$ with $\forall z(x \le z < y \to \phi_1)$.
  Now, consider the formula $\psi = \exists y (x \le y \land \phi'(x,y))$.
  Note that $\psi$ is a \coSafetyFO formula. When interpreted over
  \emph{infinite} words, the models of $\psi$ are exactly those containing a
  prefix that belongs to $\langfin(\phi)$, with the remaining suffix
  unconstrained, that is $\lang(\psi) = \langfin(\phi)\concinf$, hence
  $\semfininf{\FO}\subseteq\sem{\coSafetyFO}$, and this implies that
  $\semfininf{\cosafetyltl}\subseteq\sem{\coSafetyFO}$.

  ($\supseteq$) We know by \cref{lemma:cosafetyfoseminf} that
  $\sem{\coSafetyFO}=\semfininf{\coSafetyFO}$. Since \coSafetyFO formulas are
  also \FO formulas, we have $\sem{\coSafetyFO}\subseteq\semfininf{\FO}$. By
  \cref{th:ltlfo} and \cref{lem:ltlfinequivcosafetyltlfin}, we obtain that
  $\sem{\coSafetyFO}\subseteq\semfininf{\cosafetyltl}$.
\end{proof}
We are ready now to state the main result.
\begin{thm}
\label{th:mainth}
  $\sem{\LTL}\cap\coSAFETY=\sem{\coSafetyFO}$
\end{thm}
\begin{proof}
  We know that $\sem{\LTL}\cap\coSAFETY=\semfininf{\LTL}$ by
  \cref{lemma:cosafetytoltl}. Then, by \cref{lem:ltlfinequivcosafetyltlfin} we
  know that $\semfin{\LTL}=\semfin{\cosafetyltl}$, and this in turn implies that
  $\semfininf{\LTL}=\semfininf{\cosafetyltl}$. Since
  $\semfininf{\cosafetyltl}=\sem{\coSafetyFO}$ by
  \cref{lem:cosafetyltlinfsonesfoguardinf}, we conclude that
  $\sem{\LTL}\cap\coSAFETY=\sem{\coSafetyFO}$.
\end{proof}

This result together with \cref{thm:cosafetyltltofo} allow us to conclude the
following.
\begin{thm}
  \label{thm:final}
  $\sem{\safetyltl}=\sem{\LTL}\cap\SAFETY$
\end{thm}

Note that by Observation~\ref{obs:observationnoweaknext} and
\cref{lem:cosafetyinfcosafetynoweakfin} on one hand, and by
\cref{lemma:cosafetytoltl,lem:ltlfinequivcosafetyltlfin} on the other, the
question of whether $\sem{\safetyltl}=\sem{\LTL}\cap\SAFETY$ can be reduced to
whether $\semfininf{\cosafetyltl}=\semfininf{\cosafetyltlnoweak}$. If
\cosafetyltl and \cosafetyltlnoweak were equivalent over finite words, this
would already prove \cref{thm:final}. 
However, the next theorem states that this is not the case.

\begin{thm}
\label{thm:negativecosafetyinfinitefinite}
  $\semfin{\cosafetyltl} \neq \semfin{\cosafetyltlnoweak}$
\end{thm}
\begin{proof}
  Note that in \cosafetyltlnoweak we cannot hook the final position of the
  word without the \emph{weak tomorrow} operator. For these reasons, given
  a \cosafetyltlnoweak formula $\phi$, with a simple structural induction
  we can prove that for each $\sigma\in(2^\Sigma)^+$ such that
  $\sigma\models\phi$, it holds that $\sigma\sigma'\models\phi$ for any
  $\sigma'\in(2^\Sigma)^+$, \ie all the extensions of $\sigma$ satisfy
  $\phi$ as well. This implies that $\langfin(\phi)$ is either empty (\ie
  if $\phi$ is unsatisfiable) or infinite.  Instead, by using the
  $\emph{weak tomorrow}$ operator to hook the last position of the word, we
  can describe a finite non-empty language, for example as in the formula
  $\phi = \ltl{a & X(a & wX\false)}$. The language of $\phi$ is
  $\lang(\phi)=\set{\mathsf{aa}}$, including exactly one word, hence
  $\lang(\phi)$ cannot be described without the \emph{weak tomorrow}
  operator.
\end{proof}

Note that \cref{thm:negativecosafetyinfinitefinite} does \emph{not} contradict
\cref{thm:final}, that is, it does not imply that
$\semfininf{\cosafetyltl}\ne\semfininf{\cosafetyltlnoweak}$. For example,
consider again the formula $\ltl{a & X(a & wX\false)}$. It cannot be expressed without the \emph{weak tomorrow} operator, yet it holds that:
$
  \langfin(\ltl{a & X(a & wX\false)})\cdot(2^\Sigma)^\omega =
  \langfin(\ltl{a & Xa})\cdot(2^\Sigma)^\omega
$.


\section{The (co)safety fragment of \LTL over finite words}
\label{sec:finite}

\* 
  Scaletta:
  * Definitions of Safety and co-safety languages over finite words
  * This section is devoted to the proof of the theorem: coSafetyLTL(-wX)
    captures the cosafety fragment of LTL over finite words
  * Teorema: for any formula of FO, there exists one in coSafetyFO such
    that ...
  * Lemma:
  * Corollario:
  * Dualization of all results for the safety case.
*/

So far, we focused primarily on safety and co-safety languages of infinite
words. Naturally, safety and co-safety languages of \emph{finite words}
deserve attention as well. In this section, we define the notion of
(co-)safety languages of finite words and we prove that
\cosafetyltlnoweak (resp., \safetyltlnonext), \ie the logic
obtained from \cosafetyltl (resp., \safetyltl) by forbidding the $\ltl{wX}$
(resp., the $\ltl{X}$) operator, captures the set of co-safety (resp.,
safety) languages of \LTL interpreted over finite words.

We start with the definitions of safety and co-safety languages of finite
words, which (unsurprisingly) are the natural restriction of
\cref{def:safelang,def:cosafelang} to finite words.

\begin{defi}
  \label{def:fin:safelang}
  Let $\lang \subseteq A^*$ be a language of finite words. We say that
  $\lang$ is a \emph{safety language} if and only if for all the words
  $\sigma \in A^*$ it holds that, if $\sigma \not \in \lang$, then there
  exists an $i<|\sigma|$ such that, for all $\sigma'\in A^*$,
  $\sigma_{[0,i]}\cdot\sigma' \not \in \lang$. The class of safety
  languages of finite words is denoted as \SAFETYfin.
\end{defi}

\begin{defi}
  \label{def:fin:cosafelang}
  Let $\lang \subseteq A^*$ be a language of finite words. We say that
  $\lang$ is a \emph{co-safety language} if and only if for all the words
  $\sigma \in A^*$ it holds that, if $\sigma \in \lang$, then there exists
  an $i<|\sigma|$ such that, for all $\sigma'\in A^*$,
  $\sigma_{[0,i]}\cdot\sigma' \in \lang$. The class of co-safety languages
  of finite words is denoted as \coSAFETYfin.
\end{defi}

The remaining part of the section is devoted to the proof of the following
theorem, which gives two characterizations of the safety and co-safety
fragments of \LTL over finite words, one in terms of temporal logics and
one in terms of first-order logics.

\begin{restatable}{thm}{thfinchar}
\label{th:fin:char}
  It holds that:
  \begin{itemize}
    \item $\semfin{\LTL} \cap \coSAFETYfin = \semfin{\cosafetyltlnoweak}
      = \semfin{\coSafetyFO}$
    \item $\semfin{\LTL} \cap \SAFETYfin = \semfin{\safetyltlnonext}
      = \semfin{\SafetyFO}$
  \end{itemize}
\end{restatable}

We first prove the following auxiliary lemma.

\begin{lem}
\label{lem:fo:to:cosafetyfo}
  For any formula $\phi(x)\in\FO$ with one free variable, there exists a formula
  $\phi'(x)\in\coSafetyFO$ such that $\langfin(\phi'(x))
  = \langfin(\phi(x))\cdot(2^\Sigma)^*$.
\end{lem}
\begin{proof}
  Let $\phi(x)$ be a formula in \FO in negation normal form with one free
  variable. We define $\phi'(x)$ as the formula $\exists y \suchdot (x\le
  y \land \psi(x,y))$, where $\psi(x,y)$ is the formula with free variables
  $x$ and $y$ (where $y$ is a fresh variable that does not appear in
  $\phi(x)$) obtained from $\phi(x)$ by replacing each subformula of type
  $\exists z \suchdot \phi_1$ with $\exists z \suchdot (x \le z < y \land
  \phi_1)$ and each subformula of type $\forall z \suchdot \phi_1$ with
  $\forall z \suchdot (x \le z < y \to \phi_1)$.  It is simple to see that
  $\phi'(x)\in\coSafetyFO$ and $\langfin(\phi'(x))
  = \langfin(\phi(x))\cdot(2^\Sigma)^*$.
\end{proof}

We now prove that \coSafetyFO (interpreted over finite words) captures
$\semfinfin{\FO}$, as stated by the following Lemma.

\begin{lem}
\label{lem:cosafetyfo:finfin}
  $\semfin{\coSafetyFO} = \semfinfin{\FO}$.
\end{lem}
\begin{proof}
  We first prove the inclusion
  $\semfin{\coSafetyFO}\subseteq\semfinfin{\FO}$.  By
  \cref{lemma:cosafetyfoseminf}, it holds that
  $\semfin{\coSafetyFO}=\semfinfin{\coSafetyFO}$. Since \coSafetyFO is
  a syntactic fragment of \FO, it also holds that
  $\semfin{\coSafetyFO}\subseteq\semfin{\FO}$. It follows that
  $\semfin{\coSafetyFO}\subseteq\semfinfin{\FO}$.

  We now prove the inclusion $\semfinfin{\FO} \subseteq
  \semfin{\coSafetyFO}$. Let $\phi$ be a formula of \FO. By
  \cref{lem:fo:to:cosafetyfo}, there exists a formula $\phi'(x)$ such that
  $\langfin(\phi'(x)) = \langfin(\phi)\cdot(2^\Sigma)^*$. Since
  $\langfin(\phi)\cdot(2^\Sigma)^* \in \semfinfin{\FO}$ and
  $\langfin(\phi'(x)) \in \semfin{\coSafetyFO}$, this proves that
  $\semfinfin{\FO} \subseteq \semfin{\coSafetyFO}$.
\end{proof}

We can now prove that \coSafetyFO and \cosafetyltlnoweak capture the
co-safety fragment of \LTL interpreted over finite words, \ie
$\semfin{\LTL} \cap \coSAFETY$. By dualization, it also holds that
\SafetyFO and \safetyltlnonext are characterizations of the safety fragment
of \LTL over finite words in terms of temporal logics and first-order
logics, respectively.

\thfinchar*
\begin{proof}
  We first prove the case for the co-safety fragment. The following
  equivalences are true:
  \begin{align*}
      &\semfin{\LTL} \cap \coSAFETYfin \\
    = \quad &\semfinfin{\FO}      & \mbox{by \cref{prop:reverse,th:ltlfo,prop:galphafalpha}}\\
    = \quad &\semfin{\coSafetyFO} & \mbox{by \cref{lem:cosafetyfo:finfin}} \\
    = \quad &\semfin{\cosafetyltlnoweak} & \mbox{by \cref{cor:cosafetyfoiscosafetyltlnoweak}}
  \end{align*}
  Exploiting the duality between safety and co-safety fragments, one can
  directly obtain the proof for the safety case.
\end{proof}


\section{Comparison with related fragments}
\label{sec:comp}

In this section, we compare \coSafetyFO with two related fragments, that is
\cosafetyltl and \EBFO, another first-order logic characterization of
\LTL-definable co-safety properties. We also point out a practical
application of the translation of \cosafetyltl formulas into \coSafetyFO.
As before, all the results can be dualized to the safety case.

\subsection{Succinctness of \coSafetyFO with respect to \cosafetyltl}
\label{sub:succinct}

We show that there exists an equivalence-preserving translation from
\cosafetyltl into \coSafetyFO that involves only a linear blowup.

\begin{prop}
\label{prop:succinct}
  For all $\phi \in \cosafetyltl$, there exists $\phi' \in \coSafetyFO$
  such that:
  \begin{enumerate*}[label=(\roman*)]
    \item $\lang(\phi) = \lang(\phi')$; and
    \item $|\phi'| \in \mathcal{O}(|\phi|)$.
  \end{enumerate*}
\end{prop}
\begin{proof}
  The transformation of \cosafetyltl into \coSafetyFO is the same as the
  transformation of \cosafetyltlnoweak into \coSafetyFO maintaing the
  equivalence over finite words (see \cref{lemma:fromTLtocosafetyfo}). For
  sake of clarity, we report here the transformation.  We inductively
  define the formula $FO(\phi,x)$, where $x$ is a variable, as
  follows:\fitpar
  \begin{itemize}
    \item
      $FO(p,x) = P(x)$, for each $p \in \Sigma$
    \item
      $FO(\lnot p, x) = \lnot P(x)$, for each $p \in \Sigma$
    \item
      $FO(\phi_1 \land \phi_2,x) = FO(\phi_1,x) \land FO(\phi_2,x)$
    \item
      $FO(\phi_1 \lor \phi_2,x) = FO(\phi_1,x) \lor FO(\phi_2,x)$
    \item
      $FO(\ltl{X \phi_1},x) = \exists y(x < y \land y=x+1 \land FO(\phi_1,y))$\\
      where $y=x+1$ can be expressed as $\forall z(x<z<y \implies \false)$.
    \item 
      $FO(\ltl{\phi_1 U \phi_2},x) = {}
        \exists y( x \le y \land FO(\phi_2,y) \land 
          \forall z (x \le z < y \implies FO(\phi_1,z)))$
  \end{itemize}
  For each $\phi \in \cosafetyltlnoweak$, the formula $FO(\phi,x)$ has exactly
  one free variable $x$. 
  By the semantics of the operators in \cosafetyltl, it is immediate to see
  that for all infinite state sequences $\sigma \in (2^\Sigma)^\omega$, it
  holds that $\sigma \models \phi$ if and only if $(\sigma)^s,0 \models
  FO(\phi,x)$, and $FO(\phi,x) \in \coSafetyFO$.  Therefore,
  $\lang(FO(\phi,x)) \in \sem{\coSafetyFO}$.

  Now, we study the size of $FO(\phi,x)$ in terms of the size of $\phi$.
  From now on, let $n = |\phi|$.  If $\phi$ is an atomic formula, then
  $FO(\phi,x)$ is of constant size.  If instead $\phi \equiv
  \ltl{X}\phi_1$, then $|FO(\phi,x)| = \mathcal{O}(1) + |\phi_1|$.
  Otherwise, if $\phi \equiv \phi_1 \lor \phi_2$ or $\phi \equiv \phi_1 \land
  \phi_2$ or $\phi \equiv \phi_1 \ltl{U} \phi_2$, then without loss of
  generality we can suppose that $|\phi_1| = |\phi_2|
  = \ceil{\frac{|\phi|-1}{2}}$ and thus $|FO(\phi,x)| = \mathcal{O}(1)
  + 2 \cdot |FO(\phi_1,x)|$.

  Therefore, the size of $FO(\phi,x)$ is described by the following
  recurrence equation:
  \begin{align*}
    S(n) =
    \begin{cases}
      \mathcal{O}(1) & \mbox{if } n=1 \\
      \max\{\mathcal{O}(1) + 2 \cdot S(\frac{n}{2}), \mathcal{O}(1) + S(n-1)\} & \mbox{otherwise}
    \end{cases}
  \end{align*}
  We have that:
  \begin{align*}
    S(n) \le (2^i \cdot S(\frac{n}{2^i}) + i \cdot \mathcal{O}(1)) + 
           (S(n -1 -j) + j \cdot \mathcal{O}(1))
  \end{align*}
  For $i=\log_2(n)$ and for $j=n-2$, we obtain:
  \begin{align*}
    S(n) &\le 2^{\log_2(n)} \cdot S(\frac{n}{2^{\log_2(n)}}) + \log_2(n) \cdot
      \mathcal{O}(1) + S(n-1-n+2) + (n-2) \cdot \mathcal{O}(1) \\
    &\le n \cdot S(1) + \mathcal{O}(\log_2(n)) + S(1) + \mathcal{O}(n) \\
    &\le n \cdot \mathcal{O}(1) + \mathcal{O}(\log_2(n)) + \mathcal{O}(1) + \mathcal{O}(n) \\
    &\in \mathcal{O}(n)
  \end{align*}
  Therefore $|FO(\phi,x)| \in \mathcal{O}(|\phi|)$.
\end{proof}

Of course, also in this case, the result can be dualized, having that for
all $\phi \in \safetyltl$, there exists $\phi' \in \SafetyFO$ such that:
\begin{enumerate*}[label=(\roman*)]
  \item $\lang(\phi) = \lang(\phi')$; and
  \item $|\phi'| \in \mathcal{O}(|\phi|)$.
\end{enumerate*}

The other direction of \cref{prop:succinct} is less obvious.  The
translation of any \coSafetyFO formula into an equivalent one in
\cosafetyltl described in this paper (\cref{sec:fragments}) follows two
main steps:
\begin{enumerate*}[label=(\roman*)]
  \item the transformation of \coSafetyFO into normal form
    (\cref{lem:froms1sfoguardtonormalform});
  \item the transformation of the normal form to \cosafetyltl
    (\cref{lem:fromS1SFOguardtoTL}).
\end{enumerate*}
While the second step requires only a linear size increase, the first step, in
the general case, can produce a formula of nonelementary size with respect
to the size of the initial formula. This is mainly due to how the case of
\emph{conjunctions} is managed by the proof of
\cref{lem:froms1sfoguardtonormalform}: the resulting formula, in this case,
contains a subformula for each interleaving $\pi$ in the set of all
possible interleavings $\Pi$; since this set is exponentially larger than
the size of the starting formula, the formula resulting from the case of
conjunctions causes an exponential blow-up in the worst case.
As a consequence, the equivalence-preserving translation from \coSafetyFO
to \cosafetyltl shown in this paper is nonelementary in the size of the
final formula. 
Of course, this gives an upper bound to the succinctness of \coSafetyFO with
respect to \cosafetyltl: a still open question is about the lower bound, in
particular whether there exists a translation from any \coSafetyFO formula
to an equivalent \cosafetyltl one of polynomial size.

\subsection{A practical feedback of \coSafetyFO}
\label{sub:practice}

Interestingly, the succinctness of \coSafetyFO with respect to \cosafetyltl
described in \cref{sub:succinct} has a practical feedback in the context of
realizability and reactive synthesis.

Given a formula in \LTL over a set of \emph{controllable} and
\emph{uncontrollable} variables, realizability is the problem of
establishing whether, given any sequence $\Uncontr$ of uncontrollable
variables, there exists a strategy $s$ choosing the value of the
controllable variables in such a way to guarantee that any sequence
generated by $s$ responding to $\Uncontr$ is a model of the initial
formula. Reactive Synthesis is the problem of computing such a strategy (if
any).

In~\cite{ZhuTLPV17}, Zhu \etal consider the realizability from \safetyltl
specifications.  The first steps of their algorithm consist in negating the
starting formula (thus obtaining a formula in \cosafetyltl, after the
transformation into negation normal form), and the consequent translation
into \FO. This last step is used in order to exploit the tool
\mona~\cite{henriksen1995mona}, an efficient tool for the construction and
manipulation of automata. Interestingly, the formula resulting from this
step is a formula of \coSafetyFO of linear size with respect to the
starting one, although Zhu \etal never explicitly identified it as such.

\subsection{An alternative first-order logic characterization of (co)safety \LTL properties}
\label{sub:ebfo}

We start by giving a brief account of a different first-order logic
characterization of safety and co-safety \LTL properties, proposed by Thomas
in~\cite{thomas1988safety}.

Given a formula $\phi(x)$ in the language of \FO with one free variable
(recall \cref{sec:preliminaries}), we say that $\phi(x)$ is \emph{bounded}
if and only if all quantifiers in $\phi(x)$ are either of the form $\exists
y (y \le x \land {} \dots )$ or $\forall y (y \le x \to {} \dots)$.
The two fragments of \FO proposed by Thomas~\cite{thomas1988safety} for
capturing the safety and co-safety fragment of \LTL are defined as
follows.\footnotemark

\footnotetext{%
  Thomas did not give a name to these fragments. We chose to call them the
  \emph{Existential} and the \emph{Universal Bounded} fragment of \FO.
}

\begin{defi}
\label{def:ebfo}
  The \emph{Existential Bounded} fragment of \FO (\EBFO, for short) is the
  set of \FO sentences of type $\exists x \suchdot \phi(x)$, such that
  $\phi(x)$ is a bounded formula.
  The \emph{Universal Bounded} fragment of \FO (\UBFO, for short) is the
  set of \FO sentences of type $\forall x \suchdot \phi(x)$, such that
  $\phi(x)$ is a bounded formula.
\end{defi}
Note that, on the contrary of \coSafetyFO and \SafetyFO, formulas of \EBFO
and \UBFO do not contain any free variable. For this reason, the definition
of \emph{language} for \EBFO and \UBFO formulas differs from the case for
\coSafetyFO and \SafetyFO. We define the language of a formula $\phi$ in
\EBFO or \UBFO, denoted as $\lang(\phi)$, as the set of words
$\sigma\in(2^\Sigma)^\omega$ such that $(\sigma)^s\models\phi$.

The \EBFO and \UBFO fragments are heavily based on \FO and the
$\ltl{F\alpha}$ and $\ltl{G\alpha}$ normal forms
(\cref{prop:galphafalpha}); in particular, we recall that:
\begin{itemize}
  \item the set of \LTL-definable co-safety (resp. safety) properties is
    captured by the set of formulas of type $\ltl{F\alpha}$ (resp.
    $\ltl{G\alpha}$), where $\alpha\in\LTLFP$;
  \item by \cref{prop:reverse,th:ltlfo}, we have that $\sem{\LTLFP}
    = \sem{\FO}$.
\end{itemize}
Take for example the \EBFO fragment. The structure of its formulas
naturally resembles the $\ltl{F\alpha}$ normal form: the power of
\FO is used for representing all and only the formulas $\alpha$ in
\LTLFP, while the initial existential quantifier $\exists x \suchdot
(\ldots)$ together with the bound $\ldots \le x$ on all the other quantifiers
is used for modeling the \emph{eventually} ($\ltl{F}$) operator.
A similar rationale holds for \UBFO.
It follows that the \EBFO (resp. \UBFO) fragment is expressively complete
with respect to the co-safety (resp. safety) fragment of \LTL, that
is~\cite[Proposition 2.1]{thomas1988safety}:
\begin{itemize}
  \item $\sem{\EBFO} = \sem{\LTL} \cap \coSAFETY$
  \item $\sem{\UBFO} = \sem{\LTL} \cap \SAFETY$
\end{itemize}

\subsection{Comparison between \coSafetyFO and \EBFO}
\label{sub:comparison}

Since both \coSafetyFO and \EBFO capture the co-safety fragment of \LTL, it
follows that \EBFO and the fragment of \coSafetyFO with exactly one free
variable have the same expressive power. Clearly, the same holds for the
safety fragment, having that \UBFO and the fragment of \SafetyFO with
exactly one free variable are expressively equivalent.

We now show that, in addition of being expressively equivalent, there is
a linear-size translation between the fragment of \coSafetyFO with only one
free variable and \EBFO, and \viceversa.

\begin{prop}
\label{prop:comparison:trans}
  For any formula $\phi(x) \in \coSafetyFO$, there exists a formula $\phi'
  \in \EBFO$ such that:
  \begin{enumerate*}[label=(\roman*)]
    \item $\lang(\phi(x)) = \lang(\phi')$; and
    \item $|\phi'| \in \mathcal{O}(|\phi(x)|)$.
  \end{enumerate*}
\end{prop}
\begin{proof}
  Let $\phi(x)\in\coSafetyFO$. Since the language of the formula $\phi(x)$
  is defined as the set of state sequences that are models of $\phi(x)$
  when $x$ is interpreted as $0$ (recall \cref{sec:preliminaries}), it
  suffices to define a formula in \EBFO that extends $\phi(x)$ by forcing
  $x$ to be $0$. Formally, we define the formula $\phi'$ as follows:
  \begin{align*}
    \exists y \suchdot (\exists x \suchdot (x \le y \land \forall
    z \suchdot (z\le y \to (z< x \to \bot)) \land \phi_b(x)))
  \end{align*}
  where $\phi_b(x)$ is obtained from $\phi(x)$ by replacing any quantifier of
  the form $\exists k (\dots)$ (resp. $\forall k (\dots)$) with $\exists
  k (k \le y \land \dots)$ (resp. $\forall k (k \le y \to \dots)$).  It is
  easy to see that:
  \begin{enumerate}[label=(\roman*)]
    \item $\phi'$ is a formula of \EBFO;
    \item $\lang(\phi(x)) = \lang(\phi')$; and
    \item $|\phi'| \in \mathcal{O}(|\phi(x)|)$. \qedhere
  \end{enumerate}
\end{proof}

The converse direction holds as well.
\begin{prop}
  For any formula $\phi\in\EBFO$, there exists a formula $\phi'(x)
  \in\coSafetyFO$ such that:
  \begin{enumerate*}[label=(\roman*)]
    \item $\lang(\phi'(x)) = \lang(\phi)$; and
    \item $|\phi'(x)| \in \mathcal{O}(|\phi|)$.
  \end{enumerate*}
\end{prop}
\begin{proof}
  We first prove that, for any formula $\phi\in\EBFO$, there exists
  a formula $\phi'(x) \in\coSafetyFO$ such that:
  \begin{enumerate*}[label=(\roman*)]
    \item $\lang(\phi'(x)) = \langfin(\phi)\cdot(2^\Sigma)^\omega$; and
    \item $|\phi'(x)| \in \mathcal{O}(|\phi|)$.
  \end{enumerate*}
  We define $\phi'(x)$ as the following formula:
  \begin{align*}
    \exists y \suchdot (x\le y \land \psi(x,y))
  \end{align*}
  where $\psi(x,y)$ is the formula obtained from $\phi$ by replacing each
  subformula of type $\exists z \suchdot \phi_1$ with $\exists z \suchdot
  (x \le z \land \phi_1)$ and each subformula of type $\forall z \suchdot
  \phi_1$ with $\forall z \suchdot (x \le z < y \to \phi_1)$.  It is simple
  to see that $\phi'(x)\in\coSafetyFO$, $\lang(\phi'(x))
  = \langfin(\phi)\cdot(2^\Sigma)^\omega$ and $|\phi'(x)| \in
  \mathcal{O}(|\phi|)$.

  Now, with a simple induction, one can prove that any formula
  $\phi\in\EBFO$ is such that
  $\lang(\phi)=\langfin(\phi)\cdot(2^\Sigma)^\omega$. Therefore, we have
  that $\lang(\phi'(x))=\lang(\phi)$, which concludes the proof.
\end{proof}

The expressively equivalence between the fragment of \coSafetyFO with only
one free variable, \EBFO and the co-safety fragment of \LTL, together with
the linear-size transformation of \coSafetyFO into \EBFO
(\cref{prop:comparison:trans}), allow for the following consideration: in
order to capture the whole co-safety fragment of \LTL, it is not necessary
to have the full power of \FO, on which, as noted above, \EBFO is strongly
based; on the contrary, it suffices to use the syntax of \coSafetyFO, \ie
with existential quantifiers of type $\exists y(x < y \land \dots)$ and
with universal quantifiers of type $\forall y(x < y < z \implies \dots)$.


\section{Other Characterizations of the (co-)safety fragment of \LTL}
\label{sec:char}

In this section, we give an overview of the other characterizations that
are present in the literature of the safety and co-safety fragments of
\LTL, both on infinite and finite words.


\begin{figure}
\centering
\begin{subfigure}[b]{\textwidth}
  \centering
  \begin{tikzpicture}[
        titlerect/.style={
            rectangle,
            draw, thick,
            inner ysep=2.5ex,
            inner xsep=1.5ex,
            minimum width=5.5cm,
            minimum height=2.2cm,
            align=center,
            text width=5.5cm,
            label={[anchor=center,
                    fill=white,
                    font=\large\bfseries\sffamily]above:#1}}]
    \path 
      (0,0) node[titlerect={Temporal Modal Logics},align=left] (ltl) { 
          \begin{varwidth}{\linewidth}
          \begin{itemize}
            \item $\ltl{F\alpha}$
            \item \cosafetyltl
            \item $\semfininf{\cosafetyltlnoweak}$
          \end{itemize}
          \end{varwidth}
      }
      ++(7.0,0.0) node[titlerect={First-order Logics},align=left] (fo) { 
          \begin{varwidth}{\linewidth}
          \begin{itemize}
            \item \EBFO
            \item \coSafetyFO
            \item $\semfininf{\coSafetyFO}$
          \end{itemize}
          \end{varwidth}
      }
      ++(0,-4) node[titlerect={Automata Theory},align=left] (nfas) { 
          \begin{varwidth}{\linewidth}
          \footnotesize
          \begin{itemize}
          \setlength\itemsep{0.3ex}
            \item counter-free guarantee Streett 
            \item counter-free deterministic Occurrence B\"uchi 
            \item counter-free terminal B\"uchi 
          \end{itemize}
          \end{varwidth}
      }
      ++(-7.0,0) node[titlerect={Formal Languages Theory},align=left] (re) { 
          \begin{varwidth}{\linewidth}
          \begin{itemize}
            \item $\SF \cdot \Sigma^\omega$
            \item[\textcolor{white}{-}]
            \item[\textcolor{white}{-}]
          \end{itemize}
          \end{varwidth}
      };
  
    \draw[<->, draw=black, thick]
      (ltl) -- (fo);

    \begin{scope}[<->, thick, draw=black]
      \draw (fo) -- (nfas |- 0,-2.7);
      \draw (nfas) -- (re);
      \draw (re |- 0,-2.7) -- (ltl);
    \end{scope}
  \end{tikzpicture}
   \caption{The co-safety fragment of \LTL on infinite words semantics.}
   \label{fig:cosafety:char:inf}
\end{subfigure}

\vspace{5ex}

\begin{subfigure}[b]{\textwidth}
  \centering
  \begin{tikzpicture}[
        titlerect/.style={
            rectangle,
            draw, thick,
            inner ysep=2.5ex,
            inner xsep=1.5ex,
            minimum width=5.5cm,
            minimum height=2.2cm,
            align=center,
            text width=5.5cm,
            label={[anchor=center,
                    fill=white,
                    font=\large\bfseries\sffamily]above:#1}}]
    \path 
      (0,0) node[titlerect={Temporal Modal Logics},align=left] (ltl) { 
          \begin{varwidth}{\linewidth}
          \begin{itemize}
            \item $\ltl{F\alpha}$
            \item \cosafetyltlnoweak
            \item $\semfinfin{\cosafetyltlnoweak}$
          \end{itemize}
          \end{varwidth}
      }
      ++(7.0,0.0) node[titlerect={First-order Logics},align=left] (fo) { 
          \begin{varwidth}{\linewidth}
          \begin{itemize}
            \item \EBFO
            \item \coSafetyFO
            \item $\semfinfin{\coSafetyFO}$
          \end{itemize}
          \end{varwidth}
      }
      ++(0,-4) node[titlerect={Automata Theory},align=left] (nfas) { 
          \footnotesize
          \begin{varwidth}{\linewidth}
          \begin{itemize}
          \setlength\itemsep{0.3ex}
            \item counter-free deterministic Occurrence B\"uchi 
            \item counter-free terminal \NFAs 
            \item[\textcolor{white}{-}]
          \end{itemize}
          \end{varwidth}
      }
      ++(-7.0,0) node[titlerect={Formal Languages Theory},align=left] (re) { 
          \begin{varwidth}{\linewidth}
          \begin{itemize}
            \item $\SF \cdot \Sigma^*$
            \item[\textcolor{white}{-}]
            \item[\textcolor{white}{-}]
          \end{itemize}
          \end{varwidth}
      };
  
    \draw[<->, draw=black, thick]
      (ltl) -- (fo);

    \begin{scope}[<->, thick, draw=black]
      \draw (fo) -- (nfas |- 0,-2.7);
      \draw (nfas) -- (re);
      \draw (re |- 0,-2.7) -- (ltl);
    \end{scope}
  \end{tikzpicture}
   \caption{The co-safety fragment of \LTL on finite words semantics.}
   \label{fig:cosafety:char:fin}
\end{subfigure}

\caption{Different characterizations for the co-safety fragment of \LTL
  over (a) infinite words and (b) finite words.}
\label{fig:cosafety:char}
\end{figure}
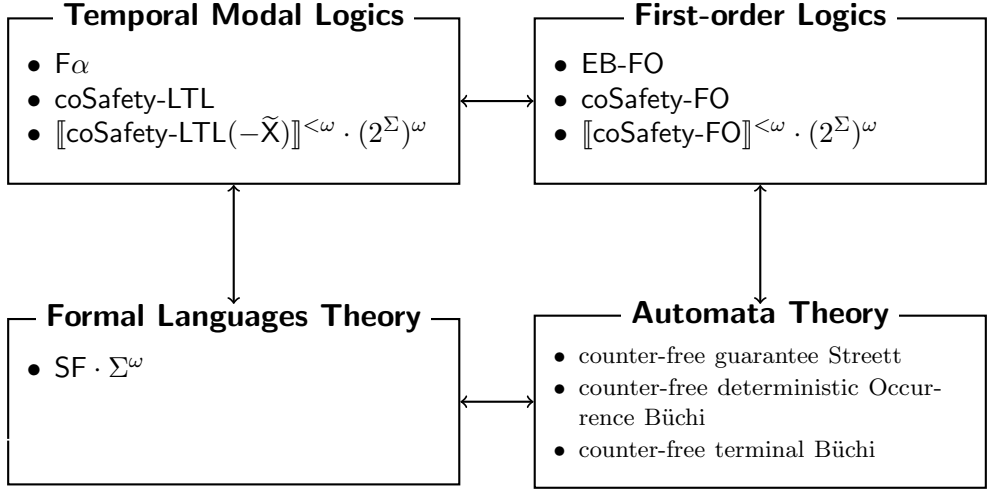
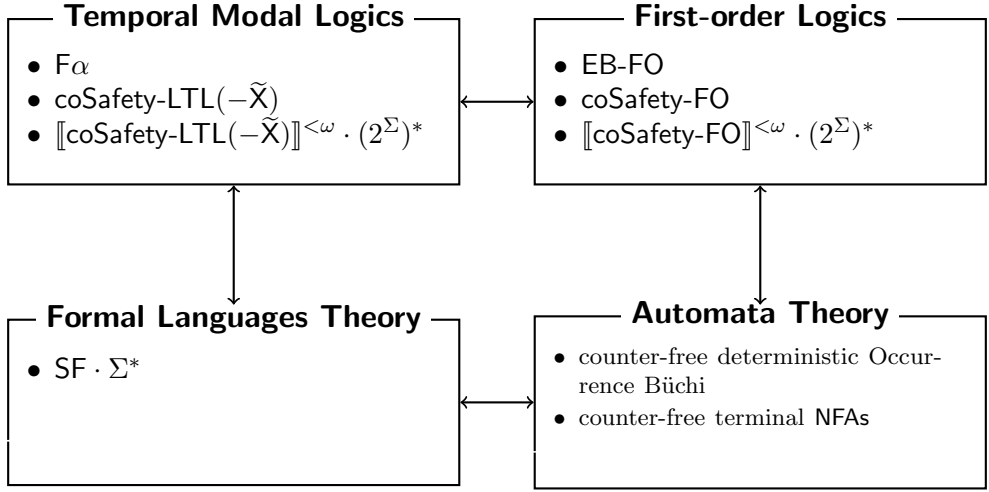


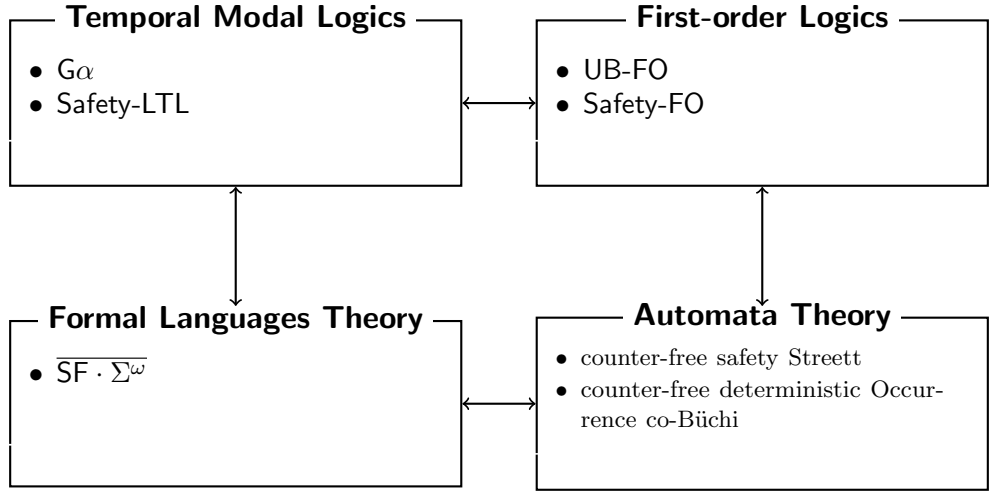
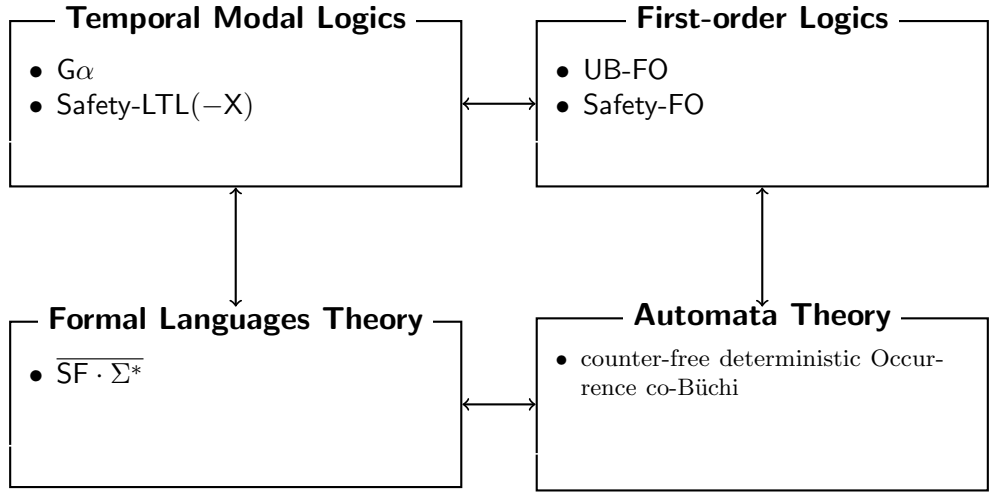
\begin{figure}
\centering
\begin{subfigure}[b]{\textwidth}
  \centering
  \begin{tikzpicture}[
        titlerect/.style={
            rectangle,
            draw, thick,
            inner ysep=2.5ex,
            inner xsep=1.5ex,
            minimum width=5.5cm,
            minimum height=2.2cm,
            align=center,
            text width=5.5cm,
            label={[anchor=center,
                    fill=white,
                    font=\large\bfseries\sffamily]above:#1}}]
    \path 
      (0,0) node[titlerect={Temporal Modal Logics},align=left] (ltl) { 
          \begin{varwidth}{\linewidth}
          \begin{itemize}
            \item $\ltl{G\alpha}$
            \item \safetyltl
            \item[\textcolor{white}{-}]
          \end{itemize}
          \end{varwidth}
      }
      ++(7.0,0.0) node[titlerect={First-order Logics},align=left] (fo) { 
          \begin{varwidth}{\linewidth}
          \begin{itemize}
            \item \UBFO
            \item \SafetyFO
            \item[\textcolor{white}{-}]
          \end{itemize}
          \end{varwidth}
      }
      ++(0,-4) node[titlerect={Automata Theory},align=left] (nfas) { 
          \footnotesize
          \begin{varwidth}{\linewidth}
          \begin{itemize}
          \setlength\itemsep{0.3ex}
            \item counter-free safety Streett 
            \item counter-free deterministic Occurrence co-B\"uchi 
            \item[\textcolor{white}{-}]
          \end{itemize}
          \end{varwidth}
      }
      ++(-7.0,0) node[titlerect={Formal Languages Theory},align=left] (re) { 
          \begin{varwidth}{\linewidth}
          \begin{itemize}
            \item $\overline{\SF \cdot \Sigma^\omega}$
            \item[\textcolor{white}{-}]
            \item[\textcolor{white}{-}]
          \end{itemize}
          \end{varwidth}
      };
  
    \draw[<->, draw=black, thick]
      (ltl) -- (fo);

    \begin{scope}[<->, thick, draw=black]
      \draw (fo) -- (nfas |- 0,-2.7);
      \draw (nfas) -- (re);
      \draw (re |- 0,-2.7) -- (ltl);
    \end{scope}
  \end{tikzpicture}
   \caption{The safety fragment of \LTL on infinite words semantics.}
   \label{fig:safety:char:inf}
\end{subfigure}

\vspace{5ex}

\begin{subfigure}[b]{\textwidth}
  \centering
  \begin{tikzpicture}[
        titlerect/.style={
            rectangle,
            draw, thick,
            inner ysep=2.5ex,
            inner xsep=1.5ex,
            minimum width=5.5cm,
            minimum height=2.2cm,
            align=center,
            text width=5.5cm,
            label={[anchor=center,
                    fill=white,
                    font=\large\bfseries\sffamily]above:#1}}]
    \path 
      (0,0) node[titlerect={Temporal Modal Logics},align=left] (ltl) { 
          \begin{varwidth}{\linewidth}
          \begin{itemize}
            \item $\ltl{G\alpha}$
            \item \safetyltlnonext
            \item[\textcolor{white}{-}]
          \end{itemize}
          \end{varwidth}
      }
      ++(7.0,0.0) node[titlerect={First-order Logics},align=left] (fo) { 
          \begin{varwidth}{\linewidth}
          \begin{itemize}
            \item \UBFO
            \item \SafetyFO
            \item[\textcolor{white}{-}]
          \end{itemize}
          \end{varwidth}
      }
      ++(0,-4) node[titlerect={Automata Theory},align=left] (nfas) { 
          \footnotesize
          \begin{varwidth}{\linewidth}
          \begin{itemize}
          \setlength\itemsep{0.3ex}
            \item counter-free deterministic Occurrence co-B\"uchi 
            \item[\textcolor{white}{-}]
            \item[\textcolor{white}{-}]
          \end{itemize}
          \end{varwidth}
      }
      ++(-7.0,0) node[titlerect={Formal Languages Theory},align=left] (re) { 
          \begin{varwidth}{\linewidth}
          \begin{itemize}
            \item $\overline{\SF \cdot \Sigma^*}$
            \item[\textcolor{white}{-}]
            \item[\textcolor{white}{-}]
          \end{itemize}
          \end{varwidth}
      };
  
    \draw[<->, draw=black, thick]
      (ltl) -- (fo);

    \begin{scope}[<->, thick, draw=black]
      \draw (fo) -- (nfas |- 0,-2.7);
      \draw (nfas) -- (re);
      \draw (re |- 0,-2.7) -- (ltl);
    \end{scope}
  \end{tikzpicture}
   \caption{The safety fragment of \LTL on finite words semantics.}
   \label{fig:safety:char:fin}
\end{subfigure}

\caption{Different characterizations for the safety fragment of \LTL
  over (a) infinite words and (b) finite words.}
\label{fig:safety:char}
\end{figure}

We start by recalling that there are four main characterizations of the set
of \LTL-definable $\omega$-languages:
\begin{itemize}
  \item in terms of temporal modal logics, $\sem{\LTL}$ is of course
    definable by \LTL and \LTLP~\cite{pnueli1977temporal};
  \item in terms of first-order logics, $\sem{\LTL}$ is captured by
    \FOTLO~\cite{kamp1968tense};
  \item in terms of regular expressions, $\sem{\LTL}$ is characterized by
    \emph{star-free} $\omega$-regular expressions~\cite{thomas1979star};
  \item in terms of automata, $\sem{\LTL}$ is captured by
    \emph{counter-free} B\"uchi automata~\cite{mcnaughton1971counter}.
\end{itemize}
Over finite words, the characterizations of \LTL are the same, except that
instead of star-free $\omega$-regular expressions and counter-free B\"uchi
automata, we consider star-free regular expressions and counter-free
nondeterministic finite automata.

In \Cref{fig:cosafety:char,fig:safety:char}, we summarize the
characterizations of the co-safety and safety fragments of \LTL, both over
infinite and finite words, in terms of:
\begin{enumerate*}[label=(\roman*)]
  \item temporal logics;
  \item first-order logics;
  \item regular expressions;
  \item automata.
\end{enumerate*}

\subsection{Temporal and first-order logics}

We first recall the characterizations in terms of temporal and first-order
logics.
In terms of temporal logics, the co-safety fragment of \LTL is captured:
\begin{itemize}
  \item over infinite words, by $\ltl{F\alpha}$, \cosafetyltl, and
    $\semfininf{\cosafetyltlnoweak}$ (\ie the finite-words interpretation
    of \cosafetyltlnoweak when concatenated to any possible infinite word);
  \item over finite words, by $\ltl{F\alpha}$, \cosafetyltlnoweak and
    $\semfinfin{\cosafetyltlnoweak}$ (\ie the finite-words interpretation
    of \cosafetyltlnoweak when concatenated to any possible finite word).
\end{itemize}
Dually, the safety fragment of \LTL is captured by $\ltl{G\alpha}$ and
\safetyltl, for the case of infinite words interpretation, and by
$\ltl{G\alpha}$ and \safetyltlnonext, for the case of finite words
interpretation.

As for first-order logics, $\sem{\LTL}\cap\coSAFETY$ (\ie the co-safety
fragment of \LTL over infinite words) is captured by \EBFO, \coSafetyFO and
$\semfininf{\coSafetyFO}$, while $\semfin{\LTL}\cap\coSAFETYfin$ (\ie the
co-safety fragment of \LTL over finite words) is captured by \EBFO,
\coSafetyFO and $\semfinfin{\coSafetyFO}$. The characterizations for the
safety fragment of \LTL over finite and infinite words is dual.

\subsection{Regular and \texorpdfstring{$\omega$}{ω}-regular expressions}

Consider now the characterization in terms of ($\omega$-)regular
expressions. We recall that a regular expression is an expression built
starting from the symbols in a finite alphabet $\Sigma$ using the
operations of union ($\cup$), complementation ($\overline{S}$),
concatenation ($\cdot$) and the Kleene'star (${}^*$). $\omega$-regular
expressions extend regular expressions by admitting also the operation
$S^\omega$, which is the $\omega$-closure of the set $S$.
\emph{Star-free} ($\omega$-)regular expressions are ($\omega$-)regular
expressions devoid of the Kleene'star.  We denote with \SF the set of
star-free regular expressions.  It is known that $\semfin{\LTL}$ (resp.
$\sem{\LTL}$) is captured by \emph{star-free} regular (resp.
$\omega$-regular) expressions~\cite{mcnaughton1971counter}.

We start with the co-safety fragment of \LTL over infinite words. Recall
that, by \cref{prop:galphafalpha}, $\sem{\LTL}\cap\coSAFETY$ is captured by
$\ltl{F\alpha}$, where $\alpha\in\LTLFP$. Moreover, by \cref{prop:reverse},
\LTLFP (\ie pure past \LTLP) is expressively equivalent to \LTL over finite
words, \ie $\semfin{\LTL} = \semfin{\LTLFP}$. Now, since $\semfin{\LTL}$ is
captured by star-free regular expressions, by the semantics of the
\emph{eventually} ($\ltl{F}$) operator, we have that
$\sem{\LTL}\cap\coSAFETY$ is captured by $\SF \cdot (\Sigma)^\omega$. For
finite words, by the same kind of reasoning, it follows that
$\semfin{\LTL}\cap\coSAFETYfin$ is captured by $\SF \cdot (\Sigma)^*$.  By
duality, $\sem{\LTL}\cap\SAFETY$ and $\semfin{\LTL}\cap\SAFETYfin$ are
captured by $\overline{\SF \cdot (\Sigma)^\omega}$ and $\overline{\SF \cdot
(\Sigma)^*}$, respectively.

\subsection{Automata}

In this part, we give an overview of some automata-based characterizations
proposed in the literature for the safety and co-safety fragments of \LTL.
We first recall some basic notions of automata theory.

\begin{defi}[Semi-automata]
\label{def:semiautom}
  A \emph{nondeterministic semi-automaton} $\autom_s$ is a tuple
  $\tuple{\Sigma,\allowbreak Q, \allowbreak q_0, \allowbreak \delta}$ such that:
  \begin{enumerate*}[label=(\roman*)]
    \item $\Sigma$ is a finite alphabet;
    \item $Q$ is a set of states;
    \item $q_0\in Q$ is the initial state;
    \item $\delta : \allowbreak Q \allowbreak \times \allowbreak \Sigma \allowbreak \to \allowbreak 2^Q$ is the transition function.
  \end{enumerate*}
\end{defi}

Given a finite alphabet $\Sigma$ and $\delta : Q \times \Sigma \to 2^Q$, we
can extend $\delta$ to $\delta^* : Q \times \Sigma^* \to 2^Q$ in the
natural way.
Given a semi-automaton $\autom_s=\tuple{\Sigma,\allowbreak Q, \allowbreak q_0, \allowbreak \delta}$, we say that
the word $\sigma\in\Sigma^*$ defines a \emph{nontrivial cycle} in $\autom$
if and only if there exists a state $q\in Q$ such that
$q\not\in\delta^*(q,\sigma)$ and $q\in\delta^*(q,\sigma^i)$ for some
$i>1$~\cite{mcnaughton1971counter,schiering1996counter}.

A semi-automaton $\autom_s=\tuple{\Sigma, \allowbreak Q, \allowbreak q_0, \allowbreak \delta}$ is said to be:
\begin{itemize}
  \item \emph{deterministic} if and only if $\delta(q,s)$ is a singleton
    set, for each $q\in Q$ and each $s\in\Sigma$.
  \item \emph{counter-free} if and only if it does \emph{not} contain any nontrivial
    cycle.
\end{itemize}

Given a semi-automaton $\autom_s=\tuple{\Sigma, \allowbreak Q, \allowbreak q_0, \allowbreak \delta}$ and a (finite
or infinite) word
$\sigma=\!\seq{\sigma_0,\sigma_1,\dots}\!\in\Sigma^*\cup\Sigma^\omega$, a run
$\pi$ over $\sigma$ is a (finite or infinite) sequence of states
$\seq{q_0,q_1, \dots}\in Q^*\cup Q^\omega$ such that $q_{i+1} \in
\delta(q_i,\sigma_i)$, for any $i\ge 0$. We denote with $\infpi(\pi)$ the
set of states that occur infinitely often in $\pi$, and with $\occpi(\pi)$
the set of states that occur at least once in $\pi$.

An automaton $\autom$ is a tuple $\tuple{\Sigma, \allowbreak Q, \allowbreak q_0, \allowbreak \delta, \allowbreak \alpha}$ such
that $\tuple{\Sigma, \allowbreak Q, \allowbreak q_0, \allowbreak \delta}$ is a semi-automaton and $\alpha$ is an
\emph{accepting condition}.
Starting from semi-automata, we can obtain many types of \emph{automata} by
defining different accepting conditions.
\begin{itemize}
  \item In nondeterministic finite automata (\NFAs, for short), $\alpha$ is
    a subset of $Q$ and is called the set of \emph{final state}. A run
    $\pi$ is accepting iff there exists an $i\ge 0$ such that $\pi_i\in
    \alpha$.
  \item In deterministic Streett automata (\DSAs, for short),
    $\alpha = \set{(G_1,R_1),\dots,(G_n,R_n)}$ where $G_i,R_i\subseteq Q$,
    for each $1\le i\le n$ and some $n\in\N$. A run $\pi$ is accepting iff,
    for all $1\le i\le n$, either $\infpi(\pi)\cap G_i = \emptyset$ or
    $\infpi(\pi)\cap R_i \neq \emptyset$.
  \item A B\"uchi automaton is a Streett automaton in which
    $\alpha=\set{(Q, R_1)}$. In this case, $R_1$ is called the set
    of \emph{final states} of the automaton.
  \item A co-B\"uchi automaton is a Streett automaton in which
    $\alpha=\set{(G_1, \emptyset)}$. In this case, $G_1$ is called the set
    of \emph{rejecting states} of the automaton.
  \item An Occurrence Streett automaton is a Streett automaton with
    accepting condition $\alpha = \set{(G_1,R_1),\dots,(G_n,R_n)}$ in which
    a run $\pi$ is accepting iff either $\occpi(\pi)\cap G_i = \emptyset$
    or $\occpi(\pi)\cap R_i \neq \emptyset$, for all $1\le i\le n$.
  \item The definitions of Occurrence B\"uchi and Occurrence co-B\"uchi
    follow from the definition of Occurrence Streett automaton.
\end{itemize}
For all types of automata, an automaton $\autom$ accepts a word $\sigma$ if
and only if there exists an accepting run induced by $\sigma$ in $\autom$.
The language recognized by $\autom$ is the set of words that are accepted
by $\autom$. It is known that each $\omega$-language definable in \LTL is
recognized by a counter-free B\"uchi automaton, and
\viceversa~\cite{mcnaughton1971counter}. Similary, a language of finite
words is definable in \LTL iff it is recognized by a counter-free
\NFA~\cite{mcnaughton1971counter}.

In~\cite{manna1990hierarchy}, Manna and Pnueli characterize the set of all
co-safety regular properties in terms of \emph{guarantee (deterministic)
Streett automata}.
A guarantee Streett automaton $\autom=\tuple{\Sigma, \allowbreak Q, \allowbreak q_0, \allowbreak \delta, \allowbreak \alpha}$
is a Streett automaton such that:
\begin{itemize}
  \item $\alpha=\set{(G_1,R_1)}$;
  \item $\goodset=(Q \setminus G_1) \cup R_1$ and $\badset=Q\setminus
    \goodset$;
  \item $\forall q \in \goodset \suchdot \forall q' \in \badset \suchdot
    \forall \sigma \in \Sigma \suchdot (q' \not\in \delta(q,\sigma))$.
\end{itemize}
Intuitively, any accepting run of a guarantee Streett automaton can visit
the states in $G_1$ only a finite number of times, after which it is forced
to visit only states in the $\goodset$ region. Crucially, once a run enters
the $\goodset$ region, each extension of it will result into an accepting
run, since it will never visit the states in $G_1$. For this reason, guarantee
Street automata capture $\REG \cap \coSAFETY$.  In order to characterize
$\sem{\LTL} \cap \coSAFETY$, by exploiting the equivalence between
counter-free automata and \LTL, Manna and Pnueli~\cite{manna1990hierarchy}
prove that \emph{counter-free guarantee Streett automata} capture the
co-safety fragment of \LTL over infinite words.
By a simple dualization, they define \emph{safety Streett automata} as
Streett automata in which there is no transition from the $\badset$ to the
$\goodset$ region, \ie:
\begin{itemize}
  \item $\alpha=\set{(G_1,R_1)}$;
  \item $\goodset=(Q \setminus G_1) \cup R_1$ and $\badset=Q\setminus
    \goodset$;
  \item $\forall q \in \badset \suchdot \forall q' \in \goodset \suchdot
    \forall \sigma \in \Sigma \suchdot (q' \not\in \delta(q,\sigma))$.
\end{itemize}
It holds that \emph{counter-free safety Streett automata} capture
$\sem{\LTL} \cap \SAFETY$.

In~\cite{CernaP03}, Cern{\'{a}} and Pel{\'{a}}nek prove that
\emph{deterministic Occurrence B\"uchi automata} are equivalent to
guarantee Streett automata, thus proving also that the formers characterize
the co-safety fragment of regular languages. The intuition behind this
characterization is simple. A run $\pi$ of a deterministic Occurrence
B\"uchi automaton is accepting if and only if it reaches a final state (say
at position $i$). Now, by definition of \emph{Occurrence B\"uchi
automaton}, every run that agrees with $\pi$ from $0$ to $i$ and then goes
on arbitrarly is accepting as well. It is not difficult to see that, in
order to capture $\sem{\LTL} \cap \coSAFETY$, it suffices to add the
counter-free condition to deterministic Occurrence B\"uchi automata.  By
dualization, Cern{\'{a}} and Pel{\'{a}}nek~\cite{CernaP03} obtain that
counter-free deterministic Occurrence co-B\"uchi automata capture
$\sem{\LTL} \cap \SAFETY$. It is simple to see that this characterization
of both the co-safety and the safety fragment of \LTL in terms of
counter-free deterministic Occurrence B\"uchi and co-B\"uchi automata holds
for finite words as well.

Last but not least, the co-safety fragment of \LTL can be captured by
\emph{counter-free terminal automata}~\cite{bloem1999efficient,CernaP03}.
Terminal automata are nondeterministic automata such that each final state
$q\in\alpha$ is such that $\delta(q,\sigma)\subseteq \alpha$ (for any
$\sigma\in\Sigma$), \ie any run, once reached a final state, cannot reach
a state which is not final. It holds that~\cite{CernaP03}:
\begin{enumerate*}[label=(\roman*)]
  \item $\sem{\LTL} \cap \coSAFETY$ is captured by counter-free terminal
    B\"uchi automata; 
  \item $\semfin{\LTL} \cap \coSAFETYfin$ is captured by counter-free
    terminal \NFAs.
\end{enumerate*}


\section{Conclusions}
\label{sec:conclusions}

In this paper, we gave a first-order characterization of safety and
co-safety languages, by means of two fragments of first-order logic,
\SafetyFO and \coSafetyFO. These fragments of \FOTLO provide a very
natural syntax and are \emph{expressively complete} with regards to
\LTL-definable safety and co-safety languages.

The core theorem establishes a correspondence between \SafetyFO (resp.,
\coSafetyFO) and \safetyltl (resp., \cosafetyltl), and thus it can be
viewed as a special version of Kamp's theorem for safety (resp., co-safety)
properties.
Thanks to these new fragments, we were able to provide a novel, compact,
and self-contained proof of the fact that \safetyltl captures
\LTL-definable safety languages. Such a result was previously proved by
Chang \etal~\cite{ChangMP92}, but in terms of the properties of
a non-trivial transformation from star-free languages to \LTL by
Zuck~\cite{zuck1986past}. As a by-product, we provided a number of results
that relate the considered languages when interpreted over finite and
infinite words. In particular, we highlighted the expressive power of the
\emph{weak tomorrow} temporal modality, showing it to be essential in
\cosafetyltl over finite words. Last but not least, we show that
\cosafetyltlnoweak and \safetyltlnonext capture the set of co-safety and
safety languages of finite words definable in \LTL, respectively.

The equivalence-preserving translation from \coSafetyFO to \cosafetyltl
shown in this paper can, in the worst case, produce formulas of
nonelementary size. An interesting future direction is to investigate
whether more efficient (even polynomial) translations are possible.

As we have seen, different fragments of \LTL can capture the (co-)safety
fragment. It is interesting to study the succinctness of these fragments,
in particular of \cosafetyltl and $\ltl{F\alpha}$, and to ask whether one
can be exponentially more succinct than the other, or whether they are
incomparable as far as succinctness is considered. Last but not least,
a natural related question is whether the previous results generalize to
the case of finite words as well, \ie for the logics \cosafetyltlnoweak and
$\ltl{F\alpha}$.

\section*{Acknowledgements}
Alessandro Cimatti, Angelo Montanari, and Stefano Tonetta acknowledge the
support of the MUR PNRR project FAIR - Future AI Research (PE00000013)
funded by the NextGenerationEU.
Luca Geatti, Nicola Gigante, and Angelo Montanari acknowledge the support
from the 2022 Italian INdAM-GNCS project \emph{``Elaborazione del
Linguaggio Naturale e Logica Temporale per la Formalizzazione di Testi''},
ref.~no.~\texttt{CUP\_E55F22000270001}.
Nicola Gigante acknowledges the support of the PURPLE project, in the
context of the AIPlan4EU project's First Open Call for Innovators.

\bibliographystyle{alphaurl}
\bibliography{biblio.bib}

\end{document}